\documentclass[10pt]{article}
\usepackage{geometry}
\usepackage[utf8]{inputenc}
\usepackage{graphicx}
\usepackage{float}
\usepackage{natbib}

\usepackage[T1]{fontenc}
\usepackage{amssymb}
\usepackage{ifthen}
\usepackage[table]{xcolor}
\usepackage{minitoc}
\usepackage{array}
\usepackage{makecell}
\usepackage{pdfpages}
\usepackage{mathtools}
\usepackage{tikz}
\usepackage{multirow}
\usepackage{booktabs}
\usepackage{enumitem}

\usepackage{arydshln} 
\usepackage{xcolor}         
\definecolor{DarkGreen}{rgb}{0.1,0.5,0.1}
\definecolor{DarkRed}{rgb}{0.5,0.1,0.1}
\definecolor{DarkBlue}{rgb}{0.1,0.1,0.5}
\definecolor{DarkYellow}{rgb}{.79,.79,0}
\definecolor{unitednationsblue}{rgb}{0.36, 0.57, 0.9}
\definecolor{blue_ppt}{rgb}{0,0.6,0.93}

\definecolor{darkblue_ppt}{rgb}{0.05,0.4,0.8}

\definecolor{orange_ppt}{rgb}{0.82,0.5,0}

\usepackage{color} 
\definecolor{yc}{RGB}{125,0,0} 
 
\definecolor{dacong}{RGB}{10,103,68}

\usepackage{hyperref}
\hypersetup{
    unicode=false,   
    pdftoolbar=true,        
    pdfmenubar=true,        
    pdffitwindow=false,     
    pdfnewwindow=true,    
    colorlinks=true,      
    linkcolor=blue,     
    citecolor=DarkGreen,       
    filecolor=DarkRed,      
    urlcolor=DarkBlue, 
    pdftitle={},
    pdfauthor={},
}
\usepackage{cleveref}
\usepackage{tablefootnote} 

\newcolumntype{?}{!{\vrule width 1pt}}
\usepackage{amsmath}
\usepackage{amsthm}
\usepackage{algorithm,algorithmic}
\definecolor{bluegray}{RGB}{77,117,154}
\newcommand{\trianglecomment}[1]{\hfill\textcolor{bluegray}{$\triangleright$ #1}}

\usepackage{graphicx} 
\usepackage{subfigure}
\usepackage{bbm}

\usepackage{extarrows}
\usepackage{bbm}
\theoremstyle{plain}
\newtheorem{lm}{Lemma} 
\newtheorem{definition}{Definition}
\newtheorem{thm}{Theorem}
\newtheorem{prop}{Proposition}
\newtheorem{asmp}{Assumption}

\newtheorem{rmk}{Remark}

\allowdisplaybreaks
\usepackage{xspace}

\usepackage{amsmath,amsfonts,bm}
\usepackage{algorithm,algorithmic}

\def\1{\bm{1}}

\DeclareMathAlphabet{\mathsfit}{\encodingdefault}{\sfdefault}{m}{sl}
\SetMathAlphabet{\mathsfit}{bold}{\encodingdefault}{\sfdefault}{bx}{n}

\def\gC{{\mathcal{C}}}
\def\gD{{\mathcal{D}}}

\def\gN{{\mathcal{N}}}
\def\gO{{\mathcal{O}}}

\def\gW{{\mathcal{W}}}
\def\gX{{\mathcal{X}}}
\def\gY{{\mathcal{Y}}}
\def\gZ{{\mathcal{Z}}}

\newcommand{\E}{\mathbb{E}}

\newcommand{\R}{\mathbb{R}}

\DeclareMathOperator*{\argmax}{arg\,max}
\DeclareMathOperator*{\argmin}{arg\,min}

\usepackage{amssymb}

\usepackage{xcolor}
\definecolor{mydarkblue}{rgb}{0,0.08,0.45}
\definecolor{mygreen}{rgb}{0.032, 0.6392, 0.2039}
\definecolor{mypurple}{HTML}{B266FF}

\def\NN{{\mathbb N}}
\def\EE{{\mathbb E}}
\def\PP{{\mathbb P}}

\def\NN{{\mathbb N}}
\def\EE{{\mathbb E}}

\newcommand{\norm}[1]{\left\| #1 \right\|}
\newcommand{\innerprod}[1]{\left\langle #1 \right\rangle}

\newcommand{\sm}{\mathsf{softmax}}

\newcommand{\ex}[2]{\mathbb{E}_{#1}\left[#2\right]}

\newcommand{\piref}{\pi_{\textnormal{ref}}}

\newcommand{\dist}[2]{D_{\textnormal{KL}}\left(#1\|#2\right)}
\newcommand{\concentr}{C_\pi}
\newcommand{\KL}[2]{\textnormal{KL}\left(#1\|#2\right)}
\newcommand{\supp}{\mathrm{supp}}

\renewcommand{\triangle}{\Delta}

\definecolor{blue_ppt}{rgb}{0,0.47,0.97}

\newcommand{\bond}{\texttt{BOND}\xspace}

\newcommand{\WIND}{\texttt{WIND}\xspace}
\newcommand{\limitibon}{\overline{\pi}^\star_{\beta}}

\title{Faster WIND: \\ Accelerating Iterative Best-of-$N$ Distillation  for LLM Alignment}

\author{Tong Yang\thanks{Carnegie Mellon University; emails: \texttt{\{tongyang,zixinw,shicongc,yuejiec\}@andrew.cmu.edu}. }  
\and Jincheng Mei\thanks{Google DeepMind; emails: \texttt{\{jcmei,hadai,schuurmans,bodai\}@google.com}.}  
\and Hanjun Dai\footnotemark[2] 
\and Zixin Wen\footnotemark[1] 
\and Shicong Cen\footnotemark[1] 
\and Dale Schuurmans\footnotemark[2]
\and Yuejie Chi\footnotemark[1] 
\and Bo Dai\footnotemark[2]  
}

\date{    \footnotemark[1]~Carnegie Mellon University\\[0.2ex]%
    \footnotemark[2]~Google DeepMind \\[2ex]%
    October 2024; revised February 2025}

\begin{document}
 
\maketitle

\begin{abstract}
    Recent advances in aligning large language models with human preferences have corroborated the growing importance of best-of-$N$ distillation (\bond). However, the iterative \bond algorithm is prohibitively expensive in practice due to the sample and computation inefficiency. This paper addresses the problem by revealing a unified game-theoretic connection between iterative \bond and self-play alignment, which unifies seemingly disparate algorithmic paradigms. Based on the connection, we establish a novel framework, \textbf{WIN} rate \textbf{D}ominance~(\WIND),
    with a series of efficient algorithms for regularized win rate dominance optimization that approximates iterative \bond in the parameter space. We provide provable sample efficiency guarantee for one of the \WIND variants with the squared loss objective. The experimental results confirm that our algorithm not only accelerates the computation, but also achieves superior sample efficiency compared to existing methods.
\end{abstract}


\tableofcontents

\section{Introduction}

Fine-tuning large language models (LLMs) to align with human preferences has become a critical challenge in artificial intelligence to ensure the safety of their deployment. Reinforcement Learning from Human Feedback (RLHF) has emerged as a dominant approach, significantly improving LLM performance as demonstrated by InstructGPT~\citep{ouyang2022training} and subsequent works. RLHF combines reward modeling to quantify human preferences and RL fine-tuning to adjust the LLM's output distribution, enhancing desired responses while suppressing unfavorable ones. While RLHF has shown promising results, it comes with significant extra post-training cost, and the aligned LLM may exhibit performance degeneration due to the alignment tax \citep{askell2021general, achiam2023gpt}.

 Alternatively, best-of-$N$ (BoN) sampling has emerged as a simple and surprisingly effective technique to obtain high-quality outputs from an LLM \citep{stiennon2020learning}. In BoN sampling, multiple samples are drawn from an LLM, ranked according to a specific attribute, and the best one is selected. This simple approach can improve model outputs without the need for extensive fine-tuning, offering a potentially more efficient path to alignment. Building upon the success of BoN sampling, a few works explored the iterative variants of this approach~\citep{dong2023raft,sessa2024bondaligningllmsbestofn}. Iterative BoN applies the BoN sampling and selection process repeatedly, potentially leading to even better alignments with human preferences.

 However, BoN incurs significant computational overhead due to making multiple inference calls to generate one output, especially when the number of calls is high.
 To mitigate the high inference cost of (iterative) BoN, \citet{sessa2024bondaligningllmsbestofn} proposed a distillation algorithm, best-of-$N$ distillation (\bond), to train a new model emulating the output of iterative BoN. However, this approach also has a high training cost, due to the need of collecting multiple samples in each round of distillation, leading to a major bottleneck for wider adoption.  
Given the growing importance and significance of the iterative BoN approach, it raises new questions about its theoretical properties, practical implementation, and relationship to established methods like RLHF. 
In this paper, we delve into the theoretical foundations and practical applications of iterative BoN sampling for LLM alignment. We address the following question:
\begin{center}
  \emph{What are the limiting points of iterative BoN, and can we design faster algorithms to find them?}
  \end{center} 

\subsection{Contributions}

We provide comprehensive answers to these questions through the following key contributions.

\begin{itemize}
  \item We introduce a general algorithmic framework for iterative BoN distillation, possibly with a slow-moving anchor, and uncover its limiting point corresponds to the Nash equilibrium of a (regularized) two-player min-max game optimizing the logarithm of the expected win rate. This offers a fresh game-theoretic interpretation that is unavailable before.
\item We show that the \textbf{WIN} rate \textbf{D}ominance (\WIND) policy, which has a higher chance of winning against any other policy, solves the
minmax game of win rate introduced in RLHF \citep{swamy2024minimaximalist,munos2023nash}, and approximates the iterative BoN's limiting point.

\item We propose a novel algorithm framework, \WIND, to find the win rate dominance policy with flexible loss configurations, and demonstrate it exhibits improved sample and computational efficiency compared to prior work, while maintaining provable convergence guarantees.

  \item We conduct extensive experiments to evaluate the performance of \WIND, demonstrating competitive or better performance against state-of-the-art alignment methods such as J-\bond across various benchmarks, highlighting its efficiency especially in the sampling process and training cost.  
\end{itemize}

\subsection{Related work}

\paragraph{RLHF and LLM alignment.} Reinforcement Learning from human feedback (RLHF) is an effective approach to train AI models to produce outputs that align with human value and preference \citep{christiano2017deep, stiennon2020learning,nakano2021webgpt}. Recently, RLHF has become the most effective approach to align language models~\citep{ouyang2022training,bai2022training}. The famous InstructGPT~\citep{ouyang2022training} approach eventually led to the groundbreaking ChatGPT and GPT-4~\citep{achiam2023gpt}. A variety of RLHF methods have been proposed, including  direct preference optimization~\citep{rafailov2024direct} and many other variants~\citep{zhao2023slic,yuan2023rrhf,azar2024general,meng2024simpo, xu2024contrastive,ethayarajh2024kto, tang2024generalized}, to name a few, which directly learns from the preference data without RL fine-tuning. Furthermore, value-incentive preference optimization~\citep{cen2024value} has been proposed to implement the  optimistic (resp. pessimistic) principle for online (resp. offline) RLHF in a practical way with theoretical guarantees.

\paragraph{RLHF via self-play.} Another line of RLHF methods investigate self-play optimization for unregularized and regularized two-player win rate games, respectively~\citep{swamy2024minimaximalist,munos2023nash}. \citet{wu2024self} introduced a scalable self-play algorithm for win rate games, enabling efficient fine-tuning of LLMs, see also \citet{rosset2024direct,zhang2024iterative,gao2024rebel,huang2024correcting} among others.

\paragraph{Best-of-$N$ and \bond.} BoN has empirically shown impressive reward-KL trade-off \citep{nakano2021webgpt,gao2023scaling}, which has been theoretically investigated by \citet{gui2024bonbon} from the perspective of win rate maximization. \citet{beirami2024theoretical} analyzed the KL divergence between the BoN policy and the base policy, and \citet{yang2024asymptotics} studied the asymptotic properties of the BoN policy. The recent work \cite{gui2024bonbon} also proposed a method to use both best-of-$N$ and worst-of-$N$ to train language models. \cite{sessa2024bondaligningllmsbestofn} introduced \bond and J-\bond to train language models to learn BoN policies. \cite{amini2024variational} proposed vBoN which is equivalent to \bond. 
However, there is no existing work for characterizing the properties of iterative BoN yet.

\paragraph{Notation.}
We let $[n]$ denote the index set $\{1, \dots, n\}$. Let $I_n$ denote the $n\times n$ identity matrix, and inner product in Euclidean space \(\mathbb{R}^n\) by $\langle\cdot,\cdot\rangle$. We let $\mathrm{supp}(\rho)$ denote the support set of the distribution $\rho$, and $\mathrm{relint}(\gC)$ represents the relative interior of set $\gC$.
We defer all the proofs to the appendix.

\section{Preliminaries}\label{sec:preliminaries}

\subsection{RLHF: reward  versus win rate }
We consider the language model $\pi_\theta(\cdot)$ as a policy, where $\theta\in\Theta$ denotes its parameters, and $\Theta$ the compact parameter space. Given a prompt $x \in \mathcal{X}$, the policy generates an answer $y \in \mathcal{Y}$ according to the conditional distribution $\pi_\theta(\cdot | x)$. For notational simplicity, we drop the subscript $\theta$ when it is clear from the context. We let $\Delta_\gY$ denote the simplex over $\mathcal{Y}$, and $\Delta_\gY^\gX$ denote the space of policies as follows:
$$\Delta_\gY^\gX\coloneqq\big\{\pi = [ \pi (\cdot |x) ]_{x\in\mathcal{X}} \mid \pi (\cdot |x) \in\Delta_\gY,\forall x\in\gX\big\}.$$

In practice, RLHF optimizes the policy model against the reward model while staying close to a reference model \(\piref\). There are two metrics being considered: reward and win rate.  
\paragraph{Reward maximization.} Suppose there is a 
reward model $r(x,y): \mathcal{X}\times \mathcal{Y} \mapsto \R$, which produces a scalar reward given a prompt \(x\) and a response \(y\). RLHF aims to maximize the KL-regularized value function,  given a reference model \(\piref\):
\begin{align}\label{eq:obj_rlhf}
    V_{\text{rm}}(\pi) = \underset{x\sim\rho,y\sim\pi(\cdot|x)}{\E}[r(x,y)] - \beta\dist{\pi}{\piref},
\end{align}
where    
$$\dist{\pi_1}{\pi_2}=\ex{x\sim\rho}{\KL{\pi_1(\cdot|x)}{\pi_2(\cdot|x)}}$$ 
is the KL divergence between policies $\pi_1$ and $\pi_2$, with $\rho$ being the distribution of prompts. Here, $\beta\geq 0$ is a hyperparameter that balances the reward and the KL divergence. Without loss of generality, we assume $\supp(\rho)$ is $\gX$ throughout the paper.

\paragraph{Win rate maximization.} Another scheme of RLHF aims to maximize  the KL-regularized win rate against the reference model~\citep{gui2024bonbon}. Given a
reward model $r$, a preference model $P_x:\gY\times\gY \rightarrow \{0, 1/2, 1\}$ can be defined as:
\begin{equation}\label{eq:P}
    P_x(y,y')\coloneqq\begin{cases}
        1, \,\,\text{ if $r(x,y)>r(x,y')$,}\\
        1/2,\,\,\text{ if $r(x,y)=r(x,y')$,}\\
        0,\,\,\text{ if $r(x,y)<r(x,y')$.}
    \end{cases}
\end{equation}
Given a policy pair $\pi,\pi'$, the \textit{win rate} of $\pi$ over $\pi'$ is given as \citep{swamy2024minimaximalist}
\begin{align}\label{eq:win_rate}
    P(\pi\succ\pi')&\coloneqq\underset{x\sim\rho,y\sim\pi(\cdot|x),\atop y'\sim\pi'(\cdot|x)}{\EE} P_x(y,y')\notag\\
    &=\E_{x\sim\rho}\pi^\top(\cdot|x)P_x(\cdot,\cdot)\pi'(\cdot|x).
\end{align}
The KL-regularized win rate maximization objective is defined as~\citep{gui2024bonbon}:
\begin{align}\label{eq:obj_win_rate}
   V_{\text{wr}}(\pi) \coloneqq P(\pi\succ\piref)- \beta\dist{\pi}{\piref}.
\end{align}
The win rate maximization is better aligned with evaluation metrics adopted in common benchmarks, and further, can be carried out without explicit reward models as long as the preference model $P_x$ is well-defined.

\subsection{Best-of-$N$ distillation} 
Best-of-$N$ (BoN) is a simple yet strong baseline in RLHF. Given a reward model $r$ and a prompt $x$, BoN samples $n$ i.i.d. responses $y_1, y_2, ..., y_n$ from the policy $\pi(\cdot|x)$ and select the response 
\[y = \arg\max_{ 1\leq i\leq n}r(x,y_i), \quad y_1, \dots, y_n \sim \pi(\cdot|x)\] 
with the highest reward. We call $\pi^{(n)}$  the BoN policy which selects the sample with the highest reward given $n$ samples i.i.d. drawn from $\pi$.
\citet{gui2024bonbon} shows that for any fixed small $\beta>0$, $\piref^{(n)}$ (approximately) maximizes $V_{\text{wr}}(\cdot)$ for properly chosen $n$.
While BoN is widely used in practice~\citep{beirami2024theoretical,gao2023scaling,wang2024transforming}, yet it can be quite expensive in terms of the inference cost for drawing $n$ samples. Hence, BoN distillation (\bond) \citep{sessa2024bondaligningllmsbestofn} is developed to approximate the BoN policy $\pi^{(n)}$ through fine-tuning  from some reference policy $\piref(\cdot|x)$, which can be updated iteratively via an exponential moving average \citep{sessa2024bondaligningllmsbestofn}.

\section{A Unified Game-Theoretic View}

In this section, we present a game-theoretic understanding of iterative BoN, which allows us to connect it to existing game-theoretic RLHF approaches under a \textit{win rate maximization} framework.

\subsection{Iterative BoN as game solving} 

\paragraph{Iterative BoN.} Due to the success of BoN sampling, its iterative version has also been studied~\citep{dong2023raft,sessa2024bondaligningllmsbestofn}, where BoN is performed iteratively by using a moving anchor as the reference policy. To understand its property in generality, we introduce the iterative BoN method in Algorithm~\ref{alg:iter_BoN} that encapsulates iterative BoN methods with or without a moving reference model, which we call the \textit{mixing} and \textit{no-mixing} case, respectively.  
\begin{algorithm}[!htb]
\caption{Iterative BoN}
\label{alg:iter_BoN}
\begin{algorithmic}[1]
\STATE \textbf{Input:} reference policy $\piref$, iterate number $T$, Best-of-$N$ parameter $n$, boolean value \textsf{Mixing}.
\STATE \textbf{Optional:} mixing rates $\alpha_1>0,\alpha_2\geq 0$ such that $\alpha_1+\alpha_2\leq 1$.
\STATE \textbf{Initialization:} $\pi_0\leftarrow\piref$.
\FOR{$t = 0,1,\cdots,T-1$} 
\STATE $\pi_t^{(n)}\gets \textrm{Best-of-$N$}(\pi_t, n)$.
\IF{\textsf{Mixing}}
    \STATE $\pi_{t+1}\propto  (\pi_t^{(n)})^{\alpha_1}\pi_t^{\alpha_2}\piref^{1-\alpha_1-\alpha_2}$; 
\ELSIF{not \textsf{Mixing}}
    \STATE $\pi_{t+1} \gets \pi_t^{(n)}$.
\ENDIF

\ENDFOR
\STATE \textbf{Return} $\pi_T$.
\end{algorithmic}
\end{algorithm}

Algorithm~\ref{alg:iter_BoN} demonstrates these two cases. In the \textit{mixing} case, we obtain new policies by mixing the BoN policy \(\pi_t^{(n)}\), $\pi_t$ and $\piref$ at each iteration with mixing rates $\alpha_1, \alpha_2$. In the \textit{no-mixing} case, the algorithm simply returns the BoN policy $\pi_t^{(n)}$ as $\pi_{t+1}$ for the next iteration. We will provide some theoretical guarantees for both cases, using the following game-theoretic framework.

\paragraph{Game-theoretic perspective.} We show that iterative BoN is implicitly solving the following game. Define a preference matrix $\overline P_x$  of size $|\gY|\times|\gY|$ at $x\in\gX$   by
\begin{equation}\label{eq:P_bar}
    \overline P_x(y,y')\coloneqq\begin{cases}
        1, \,\,\text{ if $r(x,y)\geq r(x,y')$,}\\
        0,\,\,\text{ if $r(x,y)<r(x,y')$.}
    \end{cases}
\end{equation}
Define further  $f_\beta:\Delta_\gX^\gY\times\Delta_\gX^\gY\rightarrow \R$ as
\begin{align}\label{eq:log_game_obj}
    f_\beta(\pi,\pi')&\coloneqq \underset{x\sim\rho,\atop y\in\pi(\cdot|x)}{\E}\big[\log \underset{y'\in\pi'(\cdot|x)}{\E}\overline P_x(y\succeq y')\big]- \beta\dist{\pi}{\piref}.
\end{align}
We introduce the following symmetric two-player {\em log-win-rate} game:
\begin{equation}\label{eq:log_game}
\begin{cases}
    \pi_1 =\arg\max_{\pi}f_\beta(\pi,\pi_2),&\\
    \pi_2 =\arg\max_{\pi}f_\beta(\pi,\pi_1).&
\end{cases}
\end{equation}
Let $\overline\pi^\star_\beta$ be a Nash equilibrium of the above log-win-rate game~\eqref{eq:log_game}, which satisfies the fixed-point characterization:
\begin{align}\label{eq:log_fixed_point}
    \limitibon&\in\arg\max_{\pi} \underset{x\sim\rho,\atop y\in\pi(\cdot|x)}{\E}\big[\log \underset{y'\in\limitibon(\cdot|x)}{\E}\overline P_x(y\succeq y')\big]- \beta\dist{\pi}{\piref}.
\end{align} 
 
Now we present our Theorem~\ref{thm:approx_wind}, which guarantees the convergence of Algorithm~\ref{alg:iter_BoN} to solutions of the above game. 
 \begin{thm}[Iterative BoN solves the log-win-rate game \eqref{eq:log_game}]\label{thm:approx_wind}
    Let $\piref\in \text{relint}\left(\triangle_\gY^\gX\right)$ and $n\geq 2$ in Algorithm~\ref{alg:iter_BoN}. Then $\pi_\infty\coloneqq\lim_{T\rightarrow\infty}\pi_T$ exists, and $(\pi_\infty,\pi_\infty)$ is a Nash equilibrium of the log-win-rate game~\eqref{eq:log_game} when: 
    \begin{enumerate}
        \item (no-mixing) $\alpha_{1}=1,\alpha_{2}=0$ for $\beta=0$;
        \item (mixing) $\alpha_1=\frac{\eta}{(1+\beta\eta)(n-1)}$, $\alpha_2=\frac{n-1-\eta}{(1+\beta\eta)(n-1)}$ for any $\beta,\eta>0$.
    \end{enumerate} 
\end{thm} 
In the no-mixing case, we can show that the limiting policy obtained by Algorithm~\ref{alg:iter_BoN} converges to the equilibrium of the unregularized log-win-rate game. In the mixing case, 
we show that with a proper choice of the mixing rates, Algorithm~\ref{alg:iter_BoN} solves the regularized log-win-rate game. To the best of our knowledge, this is the first game-theoretic understanding of {\em iterative} BoN using a general preference model.

\subsection{Self-play and win rate dominance}

The $\log$-win-rate game~\eqref{eq:log_game} is a non-zero-sum game that may be challenging to optimize: the function $f_\beta$ is not convex-concave, the Nash equilibria may not be unique, and the $\log$ term introduces nonlinearity, which induces difficulty in estimation.  
Therefore, we seek a good alternative to the $\log$-win-rate game that maintains its core properties while being more amenable to optimization. 

Specifically, we now consider the following alternative two-player {\em win-rate} game:
\begin{align}\label{eq:minmax}
    \max_{\pi}\min_{\pi'} P(\pi\succ\pi') &- \beta\dist{\pi}{\piref}+ \beta\dist{\pi'}{\piref},
\end{align}
which eliminates the nonlinearity in the reward, and has been recently studied by~\citet{swamy2024minimaximalist,wu2024self}
for $\beta=0$ and \citet{munos2023nash} for $\beta>0$. 

The following proposition guarantees the game \eqref{eq:minmax} is well-defined and is equivalent to the following fixed point problem:   
\begin{equation}\label{eq:optimal_policy_regularized}
  \begin{split}
    \pi^\star_\beta&\in\arg\max_{\pi}P(\pi\succ \pi^\star_\beta)- \beta\dist{\pi}{\piref}\\
&=\underset{x\sim\rho,y\sim\pi(\cdot|x),\atop y'\sim\pi^{\star}_\beta(\cdot|x)}{\EE}\left[P_x(y,y')-\beta\log\frac{\pi(y|x)}{\piref(y|x)}\right].
  \end{split}
\end{equation}

\begin{prop}[existence of $\pi^\star_\beta$]\label{prop:existence}
  $\pi^\star_\beta$ exists and $(\pi^\star_\beta,\pi^\star_\beta)$ is the Nash equilibrium of the minmax game~\eqref{eq:minmax}. Moreover, when $\beta>0$, $(\pi^\star_\beta,\pi^\star_\beta)$ is the unique Nash equilibrium.
\end{prop}

\paragraph{Win rate dominance.} The fixed-point equation \eqref{eq:optimal_policy_regularized} identifies a policy with a higher winning probability against any other policy. For $\beta=0$, $\pi^\star_0$ satisfies $P(\pi\succ\pi^\star_0)\leq1/2$ for any $\pi$, ensuring it outperforms other policies. When $\beta>0$, the KL divergence term encourages $\pi^\star_\beta$ to remain close to $\piref$ while maintaining a high win rate. We term \eqref{eq:optimal_policy_regularized} the \textit{\textbf{Win} rate \textbf{D}ominance} (\WIND) optimization problem.

\subsection{Connecting iterative BoN with \WIND}

Due to the monotonicity of $\log(\cdot)$, it is natural to postulate that the win rate game and the log-win-rate game underpinning iterative BoN are connected. We establish the novel relationship  rigorously, which allows a unifying game-theoretic view for many existing algorithms.   
We define a constant $c_\beta\in(0,+\infty]$  related to $\piref$ as:
\begin{align}\label{eq:c_beta}
c_\beta\coloneqq\min_{x\in\gX,\atop y\in\gY\backslash\gY^\star(x)}\left\{\frac{\sum_{y^\star\in\gY^\star(x)}\piref(y^\star|x)}{4\max\left\{\log\frac{\piref(y|x)}{\underset{y^\star\in\gY^\star(x)}{\max}\piref(y^\star|x)},0\right\}}\right\},
\end{align}
where $\gY^\star(x)\coloneqq\argmax_{y\in\gY} r(x,y)$ is the set of optimal responses for each $x\in\gX$. We now demonstrate the relationship between the equilibria set of the log-win-rate game $\limitibon$ and the win-rate game $\pi^\star_\beta$.
\begin{thm}[relationship between two games (informal)]\label{thm:wind_approx_relation}
    Let $\piref\in \text{relint}\left(\triangle_\gY^\gX\right)$ and $n\geq 2$ in Algorithm~\ref{alg:iter_BoN}. Then 
    \begin{itemize}
        \item When $\beta= 0$, $\limitibon$ is also a solution to \eqref{eq:optimal_policy_regularized};
        \item When $\beta\in(0,c_\beta)$ where $c_\beta$ is defined in \eqref{eq:c_beta}, for all $x\in\gX$, $\limitibon$ satisfies
        \begin{align}\label{eq:game_diff}
    \norm{\overline\pi^\star_{\beta}(\cdot | x)-\pi^\star_{\beta} (\cdot | x)}_1\leq 4(|\gY|-|\gY^\star(x)|) \exp\left(\frac{-\sum_{y^\star\in\gY^\star(x)} \piref(y^\star|x)}{4\beta}\right)\rightarrow 0\text{ as }\beta\rightarrow 0.
        \end{align}
    \end{itemize}
\end{thm}

Theorem~\ref{thm:wind_approx_relation} shows when $\beta = 0$, both games have the solution $\overline{\pi}^\star_0$. For small positive $\beta$, the $\ell_1$ distance between the solutions of the two games is bounded by a term that decreases exponentially with $1/\beta$. We verify Theorem~\ref{thm:wind_approx_relation} empirically on contextual bandits in Section~\ref{sec_app:toy_exp}.
This result provides theoretical justification for using iterative BoN as an approximation to \WIND, or vice versa, especially when $\beta$ is small. More importantly, it paves a way for efficient algorithm to \WIND, bypassing the $\log$ operator in the win-rate game.

\section{Faster WIND}

Based on the understanding of the connection between the $\log$-win-rate game and the win-rate game, in this section, we propose a new sample-efficient algorithm for finding the \WIND solution in \eqref{eq:minmax}, which includes two ingredients: {\bf (i)} identifying a memory-efficient, exact policy optimization algorithm with linear last-iterate convergence \citep{sokota2023unifiedapproachreinforcementlearning}, and {\bf (ii)} developing a series of sample-efficient algorithms with flexible loss functions and finite-time convergence guarantee. With slight abuse of terminology, we shall refer to our algorithmic framework also as \WIND.

\subsection{Exact policy optimization with last-iterate linear convergence}

Recognizing that \eqref{eq:minmax} is an KL-regularized matrix game, there are many existing algorithms that can be applied to find $\pi^\star_\beta$. Nonetheless, it is desirable to achieve fast last-iterate convergence with a small memory footprint. This is especially important in LLM optimization, for the memory efficiency.
For example, extragradient algorithms (e.g., \citet{korpelevich1976extragradient,popov1980,cen2021fast})---although fast-convergent---are expected to be expensive in terms of memory usage due to the need of storing an additional extrapolation point (i.e., the LLM) in each iteration. 

It turns out that the magnetic mirror descent algorithm in \citet{{sokota2023unifiedapproachreinforcementlearning}}, which is proposed to solve a  variational inequality formulation equivalent to \eqref{eq:minmax}, meets our consideration. We present a tailored version of this algorithm in Algorithm~\ref{alg:iter_PMDA}, and state its linear last-iterate  convergence in Theorem~\ref{thm:rate}.

\begin{algorithm}[!thb]
\caption{\WIND (exact gradient, adapted from \citet{sokota2023unifiedapproachreinforcementlearning} tailored for our setting)}
\label{alg:iter_PMDA}
\begin{algorithmic}[1]
\STATE \textbf{Input:} reference policy $\piref$, initial policy $\pi^{(0)}$, regularization coefficient $\beta>0$, learning rate $\eta>0$.
\FOR{$t = 0,1,\cdots$}
\STATE Update $\pi(\cdot|x)$ for all $x\in \gX$:
\begin{align}\label{eq:update_mirror_analytical}
  \pi^{(t+1)}(y|x)\propto(\pi^{(t)}(y|x))^{\frac{1}{1+\beta\eta}}(\piref(y|x))^{\frac{\beta\eta}{1+\beta\eta}}\exp\left(\frac{\eta}{1+\beta\eta}\E_{y'\sim\pi^{(t)}(\cdot|x)}P_x(y,y')\right)
\end{align}
\ENDFOR 
\end{algorithmic}
\end{algorithm}

\begin{thm}[Linear last-iterate convergence of Algorithm~\ref{alg:iter_PMDA}, \citet{sokota2023unifiedapproachreinforcementlearning}]\label{thm:rate}
    Assume $\beta>0$ and $\pi^{(0)},\piref\in\text{relint}(\triangle^\gX_\gY)$.
    When the learning rate $\eta\in(0,\beta]$, $\pi^{(t)}$ in Algorithm~\ref{alg:iter_PMDA} satisfies: 
    \begin{align}\label{eq:rate}
        D_{\mathrm{KL}}\big(\pi^\star_\beta||{\pi^{(t)}}\big)\leq\bigg(\frac{1}{1+\eta\beta}\bigg)^t \dist{\pi^\star_\beta}{\pi^{(0)}}.
    \end{align}
\end{thm}

\begin{rmk}\label{rmk:thm_linear_rate}
We note that when $\beta=0$, the update rule~\eqref{eq:update_mirror_analytical} recovers~\citep[Algorithm~1]{swamy2024minimaximalist}.  
When $\beta>0$, the update rule in \eqref{eq:update_mirror_analytical} is different from that of~\citet{munos2023nash}, which is
\begin{align*} 
    \pi^{(t+1)}(y|x)&\propto \widetilde\pi^{(t)} (y|x) \cdot\exp\left(\eta\EE_{y'\sim\widetilde\pi^{(t)}(\cdot|x)}P_x(y,y')\right),
\end{align*}
where $\widetilde\pi^{(t)}$ is a mixed policy defined as
$$\widetilde\pi^{(t)}(y|x)\propto \big(\pi^{(t)}(y|x)\big)^{1-\eta\beta}\left(\piref(y|x)\right)^{\eta\beta} . $$
As such, it requires extra memory to store $\widetilde\pi^{(t)}$. Moreover, \citet{munos2023nash} shows a slower rate of $\gO(1/T)$, whereas Algorithm~\ref{alg:iter_PMDA} admits linear convergence.
\end{rmk}

\subsection{Sample-efficient algorithm}\label{sec:app_framework}

 We now derive practical sample-efficient  methods for approximating the exact update \eqref{eq:update_mirror_analytical} of \WIND  in the parameter space $\Theta$ of the policy $\pi_\theta$, $\theta\in \Theta$. For exposition, we use $\phi_\theta$ to denote the logits before softmax, i.e.,
 $$\pi_{\theta}=\sm\circ\phi_\theta,$$ 
where $\sm(x)_i\coloneqq e^{x_i} /\sum_{j}e^{x_j}$ is the softmax function.

We consider the existence of reward model approximation error, i.e.,  we may use an inaccurate judger $\widehat P_x$, which is an approximation of $P_x$. For example, instead of training a reward model, we could use an LLM $\widehat P$ as a judger to directly judge if response $y$ is better than $y'$ given a prompt $x$, and use $\widehat P_x$ as an approximation of $P_x$.

\paragraph{Algorithm derivation with the squared risk.}\label{sec:practical_alg}
Let $\theta_t,\theta_\text{ref}$  denote the parameters of $\pi^{(t)}$ and $\piref$ in Algorithm~\ref{alg:iter_PMDA}, respectively.  
We rewrite the update rule~\eqref{eq:update_mirror_analytical} as
\begin{align}\label{eq:log_update}
  \phi_{\theta_{t+1}}(y|x)&=\frac{1}{1+\beta\eta}\phi_{\theta_t}(y|x) + \frac{\beta\eta}{1+\beta\eta}\phi_{\theta_{\textnormal{ref}}}(y|x)+\frac{\eta}{1+\beta\eta}\E_{y'\sim\pi_{\theta_t}(\cdot|x)}P_x(y,y')+Z_t(x)
\end{align}
for some function $Z_t:\gX\rightarrow\R$. We define a proxy $\varphi_t:\gX\times\gY\times\gY\rightarrow\R$ using the empirical win-rate as
\begin{align}\label{eq:varphi}
    \varphi_t(x,y,y')&\coloneqq \frac{1}{1+\beta\eta}\phi_{\theta_t}(y|x) + \frac{\beta\eta}{1+\beta\eta}\phi_{\theta_{\textnormal{ref}}}(y|x)+\frac{\eta}{1+\beta\eta}\widehat P_x(y,y')+Z_t(x).
\end{align}
Our main observation is that the update \eqref{eq:log_update} of $\phi_{\theta_{t}}(y|x)$ is approximating the conditional expectation of $\varphi_t$, that is, for all $(x,y)$,
$$\psi_t(x,y)\coloneqq\E_{y'\sim\pi_{\theta_t}(\cdot|x)} [\varphi_t(x,y,y')|x,y] .$$ 
Furthermore, this conditional expectation has the lowest risk, due to the following lemma:
\begin{lm}[Conditional mean minimizes the square loss]\label{lm:conditional_mean}
    For any two random variables $u,v$, we have
    \begin{equation}\label{eq:conditional_exe_regressor}
        \E_{u,v}\left[(v-\E_v(v|u))^2\right]\leq \E_{u,v}\left[(v-g(u))^2\right]
    \end{equation}
for any function $g$. In particular, the equality holds if and only if $g(u)=\E_v(v|u)$ almost everywhere on the support of the distribution of $u$.
\end{lm}

To invoke Lemma~\ref{lm:conditional_mean}, we assume the LLM is expressive enough, such that $\psi_t$ can be represented by $\phi_\theta$:
\begin{asmp}[expressive power]\label{asmp:expressive}
    For any $t\in\NN$, there exists $\theta_{t+1}^\star\in\Theta$ such that 
    \begin{equation}\label{eq:expressive}
       \forall (x,y)\in \gX\times\gY:\quad \phi_{\theta_{t+1}^\star}(y|x)=\psi_t(x,y).
    \end{equation}
\end{asmp}
Note that $\supp(\rho)=\gX$, $\supp(\pi^{(t)}(\cdot|x))=\gY$ for all $x\in\gX$, $t\in\NN$. Thus under Assumption~\ref{asmp:expressive}, by Lemma~\ref{lm:conditional_mean} we know that for all $t$, $\theta_{t+1}^\star\in\Theta$ satisfies \eqref{eq:expressive} if and only if
\begin{equation}\label{eq:theta_t+1*}
    \theta_{t+1}^\star\in\argmin_{\theta\in\Theta} R_t^{\textnormal{sq}}(\theta),
\end{equation}
where we define {\em the squared expected risk} at the $t$-th iteration $R_t^{\textnormal{sq}}(\theta)$ as
$$R_t^{\textnormal{sq}}(\theta)\coloneqq\underset{x\sim\rho,\atop y,y'\sim\pi_{\theta_t}(\cdot|x)}{\E}\left[\left(\varphi_t(x,y,y')-\phi_\theta(y|x)\right)^2\right].$$ 
In implementation, at each iteration $t$, we shall approximate $\theta_{t+1}^\star$ by minimizing the {\em empirical} risk: we sample $x_i^{(t)}\sim \rho$, $y_i^{(t)},{y_{i}'}^{(t)}\sim \pi_{\theta_t}(\cdot|x_i^{(t)})$, $i\in[M]$, and minimize the empirical risk $R^{\textnormal{sq}}_{t,M}(\theta)$ defined as 
\begin{align}\label{eq:empirical_risk}
 R_{t,M}^{\textnormal{sq}}(\theta)&\coloneqq \frac{1}{M}\sum_{i=1}^M\big(\varphi_t(x_i^{(t)},y_i^{(t)},{y_{i}'}^{(t)})-\phi_\theta(y_i^{(t)}|x_i^{(t)})\big)^2.\tag{SQ}
\end{align}

We summarize the update procedure in Algorithm~\ref{alg:practice}.

\begin{algorithm}[!htb]
\caption{\WIND (sample-efficient version)}
\label{alg:practice}
\begin{algorithmic}[1]
\STATE \textbf{Input:} reference parameter $\theta_{\textnormal{ref}}$, initial parameter $\theta_0$, regularization coefficient $\beta>0$, learning rate $\eta>0$, training set $\gD$, iteration number $T\in\NN_+$, sampling number $M\in\NN_+$.
\FOR{$t = 0,1,\cdots,T-1$}
\STATE Sample $x_i^{(t)}\sim \rho$, $y_i^{(t)},{y_{i}'}^{(t)}\sim \pi_{\theta_t}(\cdot|x_i^{(t)})$, $i\in[M]$.

\STATE $\theta_{t+1}\gets \argmin_{\theta\in\Theta} R_{t,M}(\theta). $ \trianglecomment{$R_{t,M}$ could be $R_{t,M}^{\textnormal{sq}}$, $R_{t,M}^{\textnormal{kl}}$, $R_{t,M}^{\textnormal{nce}}$, etc.}

\ENDFOR
\STATE \textbf{Return} $\theta_T$.
\end{algorithmic}
\end{algorithm}

\paragraph{Alternative risk functions.} By utilizing different variational forms, we could derive objectives different from \eqref{eq:empirical_risk}. For illustration, we provide two alternatives of $R^{\textnormal{sq}}_{t,M}(\theta)$ by using the KL divergence and the NCE loss, respectively (see Appendix~\ref{sec_app:other_objectives} 
for derivations):
    \begin{align}\label{eq:KL_loss}
        R_{t,M}^{\textnormal{kl}}(\theta)&\coloneqq-\frac{1}{M}\sum_{i=1}^M \big[\mathbbm{1}_{\{v_i=1\}}\log \zeta_\theta(x,y)+\mathbbm{1}_{\{v_i=0\}}\log (1-\zeta_\theta(x,y))\big],\tag{KL}
    \end{align}
    and
    \begin{align}\label{eq:NCE_loss}
        &R_{t,M}^{\textnormal{nce}}(\theta)\coloneqq-\frac{1}{M}\sum_{i=1}^M \Big[\left(\mathbbm{1}_{\{v_i=1\}}+\mathbbm{1}_{\{v_i'=0\}}\right)\log \frac{\zeta_\theta(x_i,y_i)}{\zeta_\theta(x_i,y_i)+p}+\left(\mathbbm{1}_{\{v_i=0\}}+\mathbbm{1}_{\{v_i'=1\}}\right)\log \frac{p}{\zeta_\theta(x_i,y_i)+p}\Big],\tag{NCE}
    \end{align}
where $v_i\sim \mathsf{Ber}(\widehat P_{x_i}(y_i,y_i'))$, $v_i'\sim \mathsf{Ber}(p)$, $p\in(0,1)$ is a hyperparameter, $\zeta_{\theta}$ is defined as
\begin{align*}
    \zeta_{\theta}(x,y)&= \frac{1+\beta\eta}{\eta}\phi_{\theta}(y|x)-\frac{1}{\eta}\phi_{\theta_t}(y|x) \notag-\beta\phi_{\theta_{\textnormal{ref}}}(y|x)-\frac{1+\beta\eta}{\eta}Z_t(x),
\end{align*}
and $\mathbbm{1}_{\{A\}}$ is the indicator function that equals 1 if $A$ is true and 0 otherwise. 

When we use the regression objective~\eqref{eq:empirical_risk}, our \WIND algorithm shares a similar form to SPPO~\citep{wu2024self}. However, \WIND differs from SPPO  in the following aspects: (i) we solve the regularized game with the KL regularization term $\beta\dist{\pi'}{\piref}$. This term is crucial in practice and is also considered in other iterative \bond methods~\citep{dong2023raft,sessa2024bondaligningllmsbestofn};  (ii) our sampling scheme is more sample-efficient: in SPPO, for each $x_i$, they sample $K$ responses $\{y_{i,j}\}_{j\in[K]}$ to estimate $\E_{y'\sim\pi_{\theta_t}(\cdot|x_i)}[P_{x_i}(y_{i,j},y')]$ by $\frac{1}{K}\sum_{k=1}^K P_{x_i}(y_{i,j},y_{i,k})$ for each $j\in[K]$. On the other hand, Lemma~\ref{lm:conditional_mean} implies estimating the conditional mean with multiple samples is unnecessary and for each $x_i$,  sampling two responses $y_i$ and $y_i'$ is enough; (iii) we allow different risk functions beyond the squared loss.  


\subsection{Convergence analysis}\label{sec:practical_convergence}

We provide a finite-sample complexity guarantee for Algorithm~\ref{alg:practice} when the risk $R_{t,M} = R_t^{\textnormal{sq}}$ is the squared loss. Our results could be easily extended to other risks. We define the reward model approximation error $\delta_P$ as
\begin{equation}\label{eq:model_error}
    \delta_P\coloneqq \max_{x\in\gX, y,y'\in\gY}\left|P_x(y,y')-\widehat P_x(y,y')\right|. 
\end{equation}

We require the following assumptions to prove the convergence of Algorithm~\ref{alg:practice}. 
The first assumes $\phi_\theta$ is differentiable and $\Theta$, $Z_t$ is bounded.
\begin{asmp}[differentiability and boundedness]\label{asmp:Lipschitz}
    The parameter space $\Theta$ is compact, $\phi_\theta(y|x)$ is differentiable w.r.t. $\theta$ for any $(x,y)\in\gX\times\gY$, and $Z_t$ in \eqref{eq:log_update} is uniformly bounded, i.e., $\exists Z\geq 0$ such that $|Z_t(x)|\leq Z$ for all $x\in\gX$ and $t\in\NN$.
\end{asmp}
Assumption~\ref{asmp:Lipschitz} guarantees the (uniform) boundedness of $\phi_{\theta}$. 
Especially, there exists $L_0>0$ such that for any $\theta,\theta'\in\Theta$ and $(x,y)\in\gX\times\gY$, we have
\begin{align}\label{eq:Lipschitz}
    |\phi_{\theta}(y|x)-\phi_{\theta'}(y|x)|\leq L_0.
\end{align}
 
Assumption~\ref{asmp:Lipschitz} also guarantees there exist $L,C>0$ such that for all $x\in\gX,y,y'\in\gY,\theta\in\Theta$ and $t$, we have
\begin{equation}\label{eq:L}
    \norm{\nabla_\theta\left[\left(\varphi_t(x,y,y')-\phi_\theta(y|x)\right)^2\right]}_2\leq L
\end{equation}
and
\begin{equation}\label{eq:M}
    \left(\varphi_t(x,y,y')-\phi_\theta(y|x)\right)^2\leq C.
\end{equation}

The next assumption controls the concentrability coefficient, which is commonly used in the RL literature, see \citet{yuan2023linear, munos2003error,munos2005error,munos2008finite,yang2023federated} for example.

\begin{asmp}[concentrability coefficient]\label{asmp:concentr}
    For Algorithm~\ref{alg:practice}, there exists finite $\concentr>0$ such that for all $t\in\NN$ and $x\in\gX$, we have
    \begin{align*}
        \E_{y\sim \piref(\cdot|x)}\left[\left(\frac{\pi^\star_\beta(y|x)}{\piref(y|x)}\right)^2\right]\leq \concentr \qquad \mbox{and}\qquad  \E_{y\sim \piref(\cdot|x)}\left[\left(\frac{\pi_{\theta_{t+1}}(y|x)}{\piref(y|x)}\right)^2\right]\leq \concentr.
        \end{align*}
\end{asmp} 

We define
\begin{align}\label{eq:C_1}
    C_1\coloneqq\exp\left( \frac{2}{\beta}\left(\delta_{P}+\frac{1+\beta\eta}{\eta}L_0+1\right)\right)C_\pi.
\end{align}

We also assume for every $t$, the expected risk $R_t$ and the empirical risk $R_{t,N}$ both satisfy Polyak-Łojasiewicz (PL) condition, which has been proven to hold for over-parameterized neural networks including transformers~\citep{liu2022loss,wu2024convergence,yang2024context}.
\begin{asmp}[Polyak-Łojasiewicz condition]\label{asmp:PL}
    For all $t\in\NN$, risk $R_t$ and empirical risk $R_{t,M}$ both satisfy the Polyak-Łojasiewicz condition with parameter $\mu>0$, i.e., for all $t\in\NN$ and $\theta\in\Theta$, we have 
    $$\frac{1}{2}\norm{\nabla_{\theta} R_t(\theta)}_2^2\geq \mu \left(R_t(\theta)-R_t(\theta_{t+1}^\star)\right)$$
    and
    $$\frac{1}{2}\norm{\nabla_{\theta} R_{t,N}(\theta)}_2^2\geq \mu \left(R_{t,M}(\theta)-R_{t,M}(\theta_{t+1})\right).$$
\end{asmp}

\begin{rmk}[Assumption~\ref{asmp:PL} is satisfied with linear function approximation]\label{rmk:linear_FA}
     We consider a special case where $\phi_\theta(y|x)=\phi(x,y)^\top\theta$ for all $(x,y)\in\gX\times\gY$, where $\phi(x,y)$ are the feature maps. If 
     for all $t\in\NN$, we have $$\E_{x\sim\rho,y\sim\pi_{\theta_t}(\cdot|x)}\big[\phi(x,y)\phi(x,y)^\top\big]\geq\frac{\mu}{2}$$
and 
$$\frac{1}{M}\sum_{i=1}^M\phi(x_i^{(t)},y_i^{(t)})\phi(x_i^{(t)},y_i^{(t)})^\top\geq\frac{\mu}{2},$$ 
then it is straightforward to verify that $R_t$ and $R_{t,M}$ are both $\mu$-strongly convex, which indicates Assumption~\ref{asmp:PL} holds~\citep{karimi2016linear}.
\end{rmk}

The following theorem gives the convergence of Algorithm~\ref{alg:practice}. 

\begin{thm}[Convergence of Algorithm~\ref{alg:practice}]\label{thm:complexity}
    Let $\theta_0=\theta_{\textnormal{ref}}$ and $\eta\in(0,\beta]$ in Algorithm~\ref{alg:practice}. Under Assumption \ref{asmp:expressive},\ref{asmp:Lipschitz},\ref{asmp:concentr},\ref{asmp:PL}, for any $T\in\NN$ and $\delta\in(0,1)$, with probability at least $1-\delta$, Algorithm~\ref{alg:practice} satisfies:
\begin{align}\label{eq:convergence_practice}
      \dist{\pi^\star_\beta}{\pi_{\theta_{T}}}&\leq \left(\frac{1}{1+\beta\eta}\right)^T\dist{\pi^\star_\beta}{\pi_{\theta_{0}}}\notag\\
      &\quad+\frac{2}{\beta}\delta_P+\frac{2(1+\beta\eta)}{\beta\eta}\sqrt{C_1 C_r\log\left(\frac{T}{\delta}\right)}\sqrt{\frac{2L^2 \log M}{\mu(M-1)} +\frac{C+2L^2/\mu}{M}},
    \end{align}
    where $C_r$ is an absolute constant, 
    $C_1$, $L$, $\delta_P$, $C$, $\mu$ are defined in 
    \eqref{eq:C_1},  \eqref{eq:L}, \eqref{eq:model_error}, \eqref{eq:M}, Assumption~\ref{asmp:PL}, respectively.
\end{thm}

Theorem~\ref{thm:complexity} indicates that, assuming no model approximation error, the total sample complexity for Algorithm~\ref{alg:practice} to reach $\varepsilon$-accuracy is 
$$2MT=\widetilde{O}\left(\left(\frac{1+\beta\eta}{\beta\eta}\right)^2\left(\frac{L^2}{\mu}+C\right)C_1 C_r\frac{1}{\varepsilon^2}\right).$$
In contrast with SPPO~\citep{wu2024self}, which only ensures average-iterate convergence without quantifying sample efficiency, our method has stronger theoretical guarantees, offering last-iterate convergence and explicit sample complexity bounds.

\section{Experiments}

We report our experimental results in this section.

\subsection{Contextual bandits}\label{sec_app:toy_exp}

In this section we conduct contextual bandit experiments to validate
Theorem~\ref{thm:wind_approx_relation}.
\paragraph{Experiments setup.} We set $|\mathcal{X}|=20,|\mathcal{Y}|=100$, and initialize $r(x_i,y_j)\overset{i.i.d}{\sim}\gN(0,1)$, where $i\in[|\mathcal{X}|],j\in[|\mathcal{Y}|]$, and $\gN(0,1)$ stands for the standard Gaussian distribution. We set $\piref$ and $\rho$ to be uniform distributions and randomly initialized $\pi^{(0)}$ in Algorithm~\ref{alg:iter_PMDA} using the Dirichlet distribution with parameters all set to be 1. For the distance metric, we use the average $\ell_1$ distance $D_{\ell_1}:\Delta^{\gX}_{\gY}\times \Delta_{\gY}^{\gX}\rightarrow \R$ defined as
\begin{align}\label{eq:average_l1_dist}
    D_{\ell_1}(\pi,\pi')\coloneqq \E_{x\sim\rho}\norm{\pi_x-\pi'_x}_1.
\end{align}
We conduct the following two experiments:
\begin{itemize}
    \item For the no-mixing case where $\alpha_1=1,\alpha_2=0$, we show the convergence of both iterative BoN (c.f. Algorithm~\ref{alg:iter_BoN}) and exact \WIND (c.f. Algorithm~\ref{alg:iter_PMDA}) to $\overline\pi^\star_0$: we plot the average $\ell_1$ distance between $\overline\pi^\star_0$ and the iterates for both algorithms. In this experiments we set learning rate $\eta$ in Algorithm~\ref{alg:iter_PMDA} to be 16.
    
    \item For the mixing case where $\alpha_1=\frac{\eta}{(1+\beta\eta)(n-1)}$ and $\alpha_2=\frac{n-1-\eta}{(1+\beta\eta)(n-1)}$, we verify that $\limitibon$ and $\pi^\star_\beta$ are very close to each other: we fix the iteration number $T=5000$ for both Algorithm~\ref{alg:iter_BoN} and \ref{alg:iter_PMDA}, and increase $\beta$ from 0.01 to 0.1 to plot the change of average $\ell_1$ distance between the final outputs of the two algorithms $D_{\ell_1}(\pi_T,\pi^{(T)})$ with respect to $\beta$. In this experiments we set $\eta=1$.
\end{itemize}

\begin{figure}[ht]
    \centering
    \begin{minipage}[c]{0.45\textwidth}
  \subfigure[no-mixing]{\label{subfig:no_mixing}
  \includegraphics[width=\linewidth]{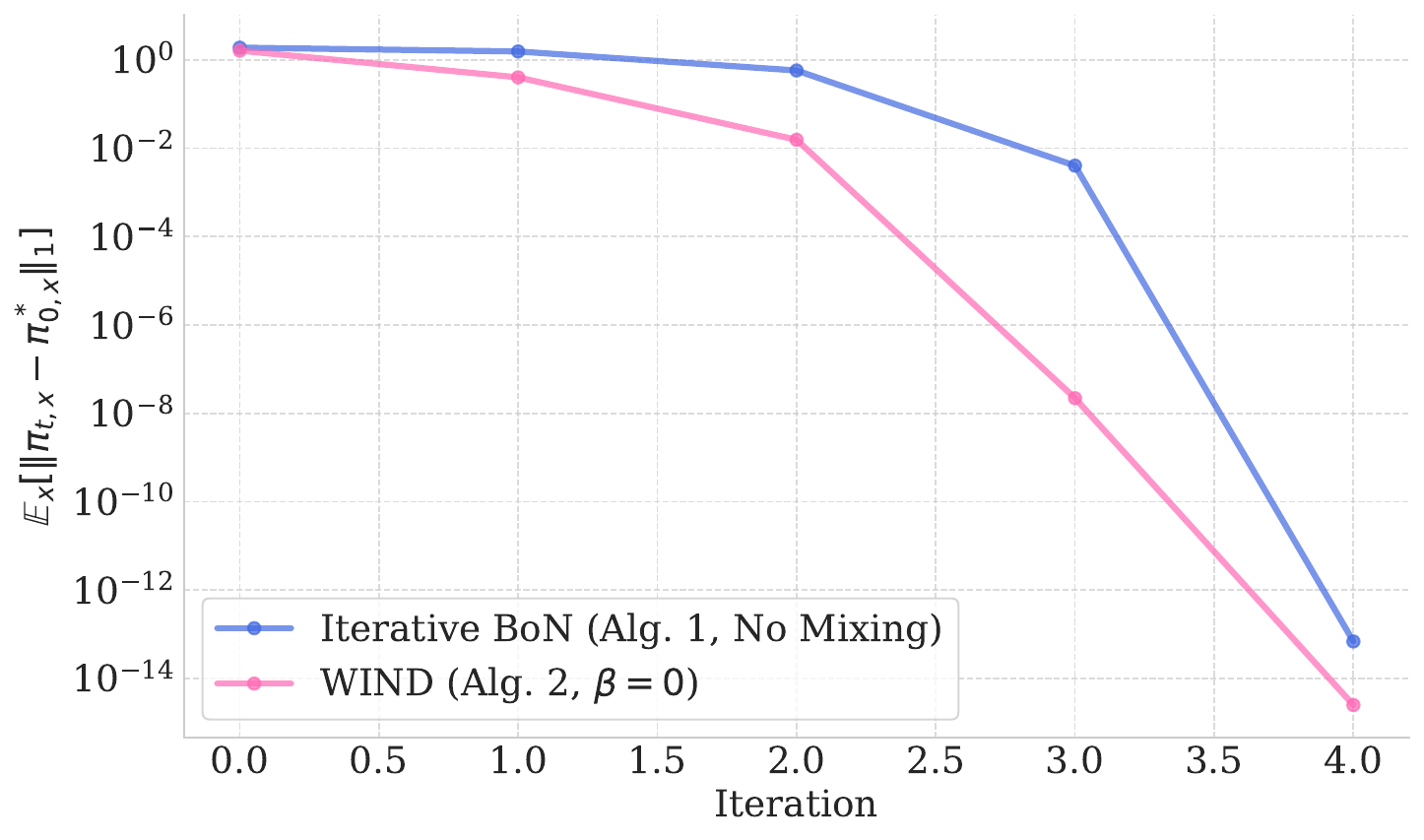}}
  \end{minipage}\hfill
  \begin{minipage}[c]{0.45\textwidth}
  \subfigure[mixing]{\label{subfig:mixing}
  \includegraphics[width=\linewidth]{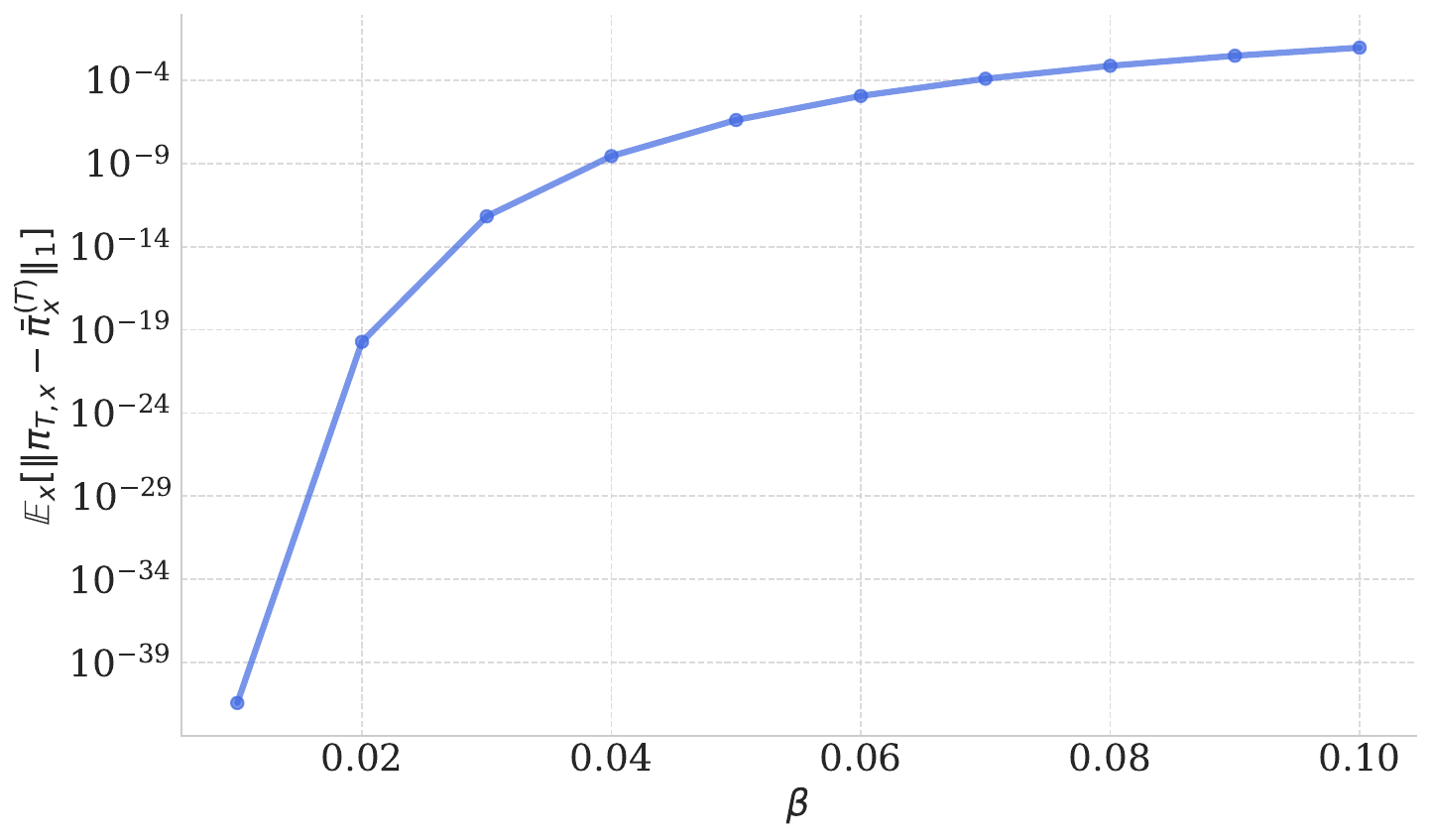}}
  \end{minipage}\hfill
    \caption{Empirical validation of Theorem~\ref{thm:wind_approx_relation} on contextual bandit experiments. For (a) the no-mixing case, we show the convergence of both iterative BoN (c.f. Algorithm~\ref{alg:iter_BoN}) and exact \WIND (c.f. Algorithm~\ref{alg:iter_PMDA}) to $\overline\pi^\star_0$; for (b) the mixing case, we show $\limitibon$ and $\pi^\star_\beta$ are very close to each other.}
    \label{fig:toy}
\end{figure}
\paragraph{Results.} Our results are presented in Figure~\ref{fig:toy}.
Figure~\ref{subfig:no_mixing} indicates that for the no-mixing case, both algorithms converge to $\overline\pi^\star_0$ with \WIND slightly faster than iterative BoN.  From Figure~\ref{subfig:mixing}, we can see that $\limitibon$ and $\pi^\star_\beta$ are very close to each other when $\beta$ is small and their distance approaches 0 very quickly as $\beta\rightarrow 0$, which corroborates~\eqref{eq:game_diff}.

\begin{table*}[t]
\centering
\begin{tabular}{c|c|c|c|ccc} 
\toprule
\multirow{2}{*}{Model} & \multirow{2}{*}{GSM8k} & \multirow{2}{*}{HellaSwag} & \multirow{2}{*}{MMLU} & \multicolumn{3}{c}{MT-Bench} \\  
                                &&&& 1st Turn & 2nd Turn  & Avg \\
\midrule
Llama-3-8B-SPPO Iter1 & 75.44 & 79.80 & 65.65 & 8.3813 & 7.6709 & 8.0283 \\ 
Llama-3-8B-SPPO Iter2 & 75.13 & 80.39 & 65.67 & 8.3875 & 7.4875 & 7.9375 \\ 
Llama-3-8B-SPPO Iter3 & 74.91 & \textbf{80.86} & 65.60 & 8.0500 & 7.7625 & 7.9063 \\ 
\midrule
Llama-3-8B-JBOND Iter1 & 76.12 & 77.70 & 65.73 & 8.2875 & 7.4281  & 7.8578 \\ 
Llama-3-8B-JBOND Iter2 & 75.74 & 77.47 & 65.85 & 8.2563 & 7.4403  & 7.8495 \\ 
Llama-3-8B-JBOND Iter3 & 76.12 & 77.36 & 65.83 & 8.2750 & 7.2767 & 7.7774 \\ 
\midrule
\rowcolor{lightgray}
Llama-3-8B-WIND Iter1 (Ours) & 75.82 & 78.73 & 65.77 & 8.2875 & 7.6875 & 7.9875 \\ 
\rowcolor{lightgray}
Llama-3-8B-WIND Iter2 (Ours) & 76.19 & 79.05 & 65.77 & 8.3625 & 7.7500 & 8.0563 \\ 
\rowcolor{lightgray}
Llama-3-8B-WIND Iter3 (Ours) & \textbf{77.18} & 79.31 & \textbf{65.87} & \textbf{8.5625} & \textbf{7.8354} & \textbf{8.2013} \\ 
\bottomrule
\end{tabular}
\caption{Results on GSM8k, HellaSwag, MMLU and MT-Bench.}
\label{tab:llama3-8b-results-1}
\end{table*}

\subsection{LLM alignment}
We follow the experiment setup in \cite{wu2024self} and Snorkel\footnote{\url{https://huggingface.co/snorkelai/Snorkel-Mistral-PairRM-DPO}}. We use Llama-3-8B-Instruct\footnote{\url{https://huggingface.co/meta-llama/Meta-Llama-3-8B-Instruct}} as the base pretrained model for baseline comparisons. For fair comparison, we chose the same prompt dataset UltraFeedback \citep{cui2023ultrafeedback} and round splits, and the same Pair-RM framework \citep{jiang2023llm} for the preference model as in \cite{wu2024self} and Snorkel. The learning rate is set to be $5 \times 10^{-7}$. In each iteration, we generate answers from $20000$ prompts in the UltraFeedback dataset to train the model. The global training batch size is $64$ ($4$ per device $\times$ $16$ GPUs). Our experiments are run on $16$ A$100$ GPUs, where each has $40$ GB memory. We modify the per-device batch size and gradient accumulation steps in SPPO GitHub repository\footnote{\url
{https://github.com/uclaml/SPPO}} while keeping the actual training batch size, to avoid out-of-memory error.

\paragraph{Baselines and Benchmarks.}
We consider two baselines: SPPO \citep{wu2024self} and a variant of J-\bond \citep{sessa2024bondaligningllmsbestofn}. Here we follow the exact same setting in their repository of the SPPO paper to reproduce SPPO results, with the only change being that we use different computation devices.

We consider 4 major evaluation benchmarks: GSM8k, HellaSwag, MMLU and MT-Bench. They evaluated the following capability:
\begin{itemize}
    \item GSM8k \citep{cobbe2021training} evaluates the mathematical reasoning at a grade school level.
    \item HellaSwag \citep{zellers2019hellaswag} measures the commonsense reasoning by letting language models select a choice to finish a half-complete sentence.
    \item MMLU \citep{hendrycks2020measuring} is a large-scale benchmark that encompasses a variety of tasks to measure the language models' knowledge.
    \item MT-Bench \citep{zheng2023judging} is also a LLM-as-a-judge benchmark that evaluates the LLM's multi-round chat capability. The scores given by GPT-4 is reported.
\end{itemize}

\paragraph{Results.} For traditional benchmarks (GSM8k, HellaSwag and MMLU), which do not involve using LLMs as the judges, the results are shown in Table~\ref{tab:llama3-8b-results-1}. The model Llama-3-8B-WIND of ours achieved optimal in the last iteration in GSM8k and MMLU, while performing better than the J-\bond variant and slightly worse than SPPO in HellaSwag. In fact, our method shows consistent improvement over iterations: for all three benchmarks, our method continues to improve with more iterations of training, while both SPPO and J-\bond variant show performance regressions with increasing number of iterations. For MT-Bench, Llama-3-8B-WIND achieves comparable results in comparison with SPPO, and outperforms J-\bond.
\begin{figure}[ht]
    \centering
        \centering
        \includegraphics[width=0.5\linewidth]{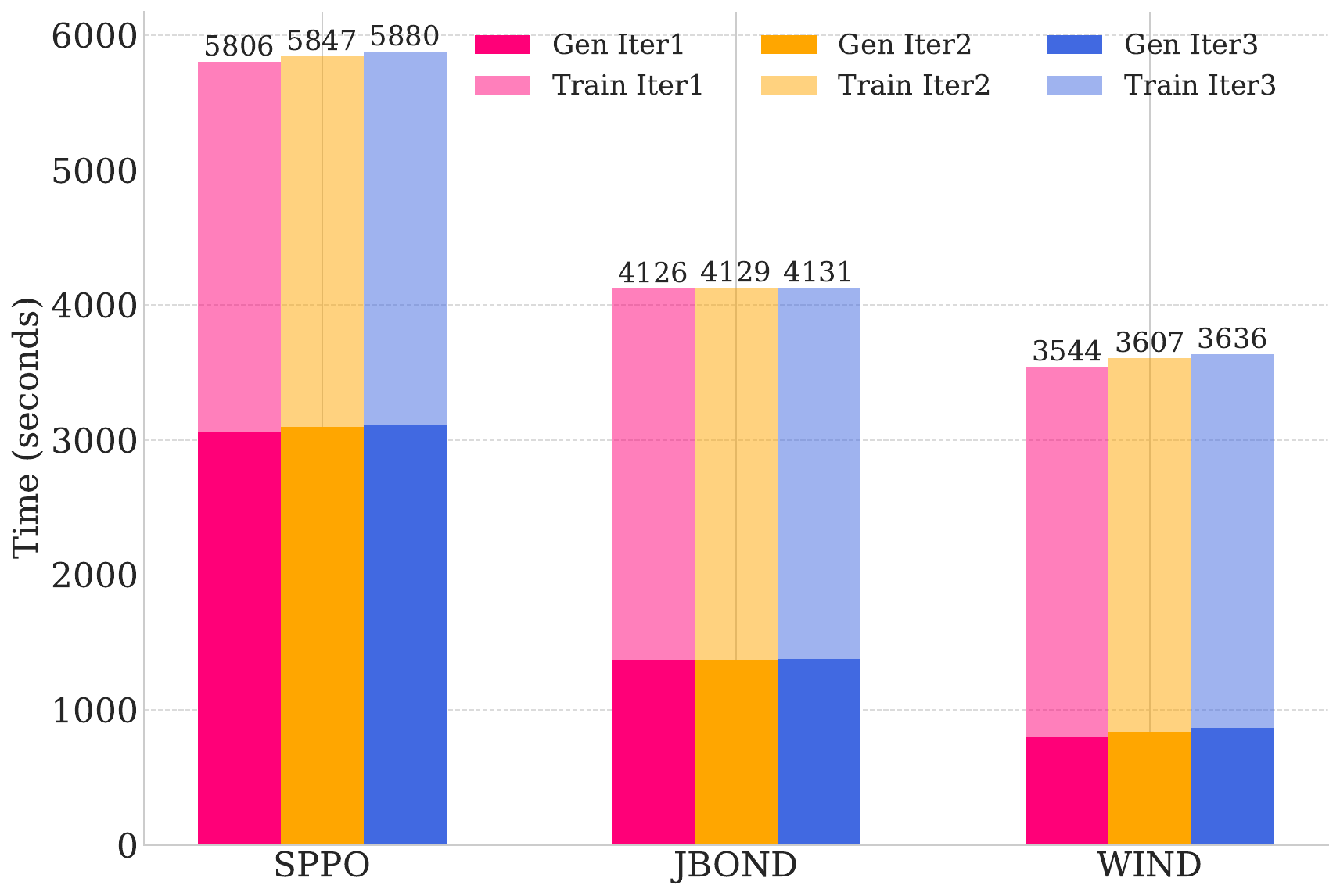}
        \caption{Run time (seconds) of different methods. }
        \label{tab:llama3-8b-running-time}
        \vspace{-0.2in}
\end{figure}
\paragraph{Runtime.} We also report the running time used by different methods in our setting. Since we base our implementation on the SPPO GitHub Repository, we only modify the objectives and the sampling process to reflect the difference between these algorithms. Figure~\ref{tab:llama3-8b-running-time} shows that our method achieves much better sample efficiency during data generation.
In sum, the proposed \WIND achieves superior performance with less computation cost, making iterative \bond practice applicable.

\section{Conclusion}

This work establishes a unified game-theoretic framework that connects iterative BoN with existing game-theoretic RLHF approaches. We present \WIND, a sample-efficient efficient algorithm for win rate dominance optimization with finite-sample guarantees, which provides an accelerated alternative to iterative \bond. Empirical validation on multiple public benchmarks demonstrates the effectiveness and efficiency of our approach compared to several state-of-the-art methods. In the future, it is interesting to explore schemes that incorporate exploration in the game-theoretic framework when the payoff matrix is accessed using bandit feedbacks \citep{yang2025incentivize}.

\section*{Acknowledgement}
 The work of T. Yang, Z. Wen, S. Cen and Y. Chi is supported in part by the grants NSF CIF-2106778, DMS-2134080 and ONR N00014-19-1-2404. In addition, the authors thank Xingyu Xu for the great discussion on the proof details.

{
\bibliographystyle{abbrvnat}
\bibliography{biblio}
}
 
\appendix

\appendix


\section{Other Objectives for Algorithm~\ref{alg:practice}}\label{sec_app:other_objectives}
In this section we give some possible alternative objectives for Algorithm~\ref{alg:practice} by utilizing different variational forms.

\paragraph{$f$-divergence objectives:} We could use $f$-divergence $D_f$ as the objective function. For a convex function $f$, $D_f$ is defined as
\begin{align}\label{eq:f_divergence}
    D_f(P||Q)\coloneqq \E_Q \left[f\left(\frac{P}{Q}\right)\right].
\end{align}
Let $$U=(x,y)\quad \text{and} \quad V=\begin{cases}
    1, \text{ if $r(x,y)> r(x,y')$,}\\
    0, \text{ if $r(x,y)<r(x,y')$,}\\
    z\sim Ber(1/2), \text{ if $r(x,y)=r(x,y')$,}
\end{cases}$$
Then $V|U$ is a function of $y'$ and $P_{V|U}=\mathsf{Ber}(\E_{y'}P_x(y,y'))$. Further define 
\begin{equation}\label{eq:psi}
    \zeta_{\theta}(x,y)= \frac{1+\beta\eta}{\eta}\phi_{\theta}(y|x)-\frac{1}{\eta}\phi_{\theta_t}(y|x) -\beta\phi_{\theta_{\textnormal{ref}}}(y|x)-\frac{1+\beta\eta}{\eta}Z_t(x).
\end{equation}
We let $Q_{V|U}=\mathsf{Ber}(\zeta_\theta(y|x))$, then by solving
\begin{equation}\label{eq:approx_update_f}
    \theta_{t+1}=\arg\min_{\theta}\E_U D_f(P_{V|U}||Q_{V|U}),
\end{equation}
we could approximate the update rule~\eqref{eq:log_update}.

In particular, when $f(x)=x\log x$, we have $D_f=D_{\text{KL}}$, and \eqref{eq:approx_update_f} becomes
\begin{align}\label{eq:loss_KL}
    \theta_{t+1}&=\arg\min_{\theta}\E_{x\sim\rho,y\sim\pi_{\theta_t}(\cdot|x)}\E_{v\sim P_{V|U}}\log\frac{P_{V|U}(v)}{Q_{V|U}(v)}\notag\\
    &=\arg\min_{\theta}\E_{x\sim\rho,y\sim\pi_{\theta_t}(\cdot|x)}\E_{v\sim P_{V|U}}[-\log Q_{V|U}(v)].
\end{align}
We could approximate the above objective by sampling $x_i\sim\rho,y_i,y_i'\sim\pi_{\theta_t}(\cdot|x_i)$ ($i\in[M]$) and minimizing the empirical risk:
\begin{align}\label{eq:loss_KL_empirical}
    \theta_{t+1}=\arg\min_{\theta}R_{t,M}^{\textnormal{kl}}(\theta)\coloneqq-\frac{1}{M}\sum_{i=1}^M \left[\mathbbm{1}_{\{v_i=1\}}\log \zeta_\theta(x,y)+\mathbbm{1}_{\{v_i=0\}}\log (1-\zeta_\theta(x,y))\right],
\end{align}
where 
\begin{equation}\label{eq:v}
    v_i=\begin{cases}
        1, \text{ if $r(x_i,y_i)> r(x_i,y_i')$,}\\
        0, \text{ if $r(x_i,y_i)<r(x_i,y_i')$,}\\
        z\sim Ber(1/2), \text{ if $r(x_i,y_i)=r(x_i,y_i')$.}
    \end{cases}
\end{equation}
For other $f$-divergence objectives, we may not be able to get rid of the unknown $P_{V|U}$ on the RHS of \eqref{eq:approx_update_f}, but \eqref{eq:approx_update_f} could provide a gradient estimator for the objective that allows us to optimize $\theta$ by stochastic gradient descent.

\paragraph{Noise contrastive estimation (NCE) objectives.} We could use NCE~\citep{gutmann2010noise} as objectives. NCE is a method to estimate the likelihood of a data point by contrasting it with noise samples. Let $D_\theta$ be the discriminator (parameterized by $\theta$) that distinguishes the true data from noise samples. The NCE objective is
\begin{equation}\label{eq:nce}
    \min_{\theta}\E_{z\sim P_{\textnormal{data}}}[-\log D_\theta(z)]+\E_{z\sim P_{\textnormal{noise}}}[-\log(1-D_\theta(z))],
\end{equation}
where the optimal solution of \eqref{eq:nce} is
\begin{align*}
    D_{\theta^\star}(z)=\frac{P_{\textnormal{data}}(z)}{P_{\textnormal{data}}(z)+P_{\textnormal{noise}}(z)}.
\end{align*}
In our case, we let $P_{\textnormal{data}}=P_{V|U}$ and $P_{\textnormal{noise}}=\mathsf{Ber}(p)$ where $p\in(0,1)$ is a hyperparameter. We also let 
$$D_\theta(1|x,y)=\frac{\zeta_\theta(x,y)}{\zeta_\theta(x,y)+p},\quad\text{and}\quad D_\theta(0|x,y)=\frac{p}{\zeta_\theta(x,y)+p},$$
where $\zeta_\theta$ is defined in \eqref{eq:psi}. Then we could approximate the update rule~\eqref{eq:log_update} by solving
\begin{align}\label{eq:nce_update}
    \theta_{t+1}&=\arg\min_{\theta}\E_{x\sim\rho,y\sim\pi_{\theta_t}(\cdot|x)}\E_{v\sim P_{V|U}}\left[-\mathbbm{1}_{\{v=1\}}\log \frac{\zeta_\theta(x,y)}{\zeta_\theta(x,y)+p}-\mathbbm{1}_{\{v=0\}}\log \frac{p}{\zeta_\theta(x,y)+p}\right]\notag\\
    &\quad\quad\quad\,\,+\E_{x\sim\rho,y\sim\pi_{\theta_t}(\cdot|x)}\E_{v\sim \mathsf{Ber}(p)}\left[-\mathbbm{1}_{\{v=1\}}\log \frac{p}{\zeta_\theta(x,y)+p}-\mathbbm{1}_{\{v=0\}}\log \frac{\zeta_\theta(x,y)}{\zeta_\theta(x,y)+p}\right].
\end{align}

The sample version of \eqref{eq:nce_update} becomes 
\begin{align}\label{eq:nce_update_sample}
    \theta_{t+1}=\arg\min_{\theta}R_{t,M}^{\textnormal{nce}}(\theta)\coloneqq-\frac{1}{M}\sum_{i=1}^M &\bigg[\left(\mathbbm{1}_{\{v_i=1\}}+\mathbbm{1}_{\{v_i'=0\}}\right)\log \frac{\zeta_\theta(x_i,y_i)}{\zeta_\theta(x_i,y_i)+p} \nonumber\\
    &  \quad +\left(\mathbbm{1}_{\{v_i=0\}}+\mathbbm{1}_{\{v_i'=1\}}\right)\log \frac{p}{\zeta_\theta(x_i,y_i)+p}\bigg],
\end{align}
where $v_i$ is defined in \eqref{eq:v} and $v_i'\sim \mathsf{Ber}(p)$.

\section{Proofs}\label{sec_app:proofs} 

For notational simplicity, for any policy $\pi\in \Delta_\gY^\gX$, we denote $\pi_x = \pi(\cdot |x)$ throughout the proof.

\subsection{Proof of Proposition~\ref{prop:existence}}\label{sec_app:proof_prop_exist}
We first prove the case when $\beta=0$. 
This part of proof is inspired by~\citet[Lemma~2.1]{swamy2024minimaximalist}.
Let
\begin{align*}
        \pi_1\coloneqq\arg\max_{\pi}\min_{\pi'} P(\pi\succ\pi'),\quad\pi_2\coloneqq\arg\min_{\pi'}\max_{\pi} P(\pi\succ\pi'),
\end{align*}
i.e., $(\pi_1,\pi_2)$ is a Nash equilibrium of \eqref{eq:minmax} (which is guaranteed to exist since the policy space is compact).
Then $\forall \pi,\pi'$ and $\forall x\in \gX$, we have:
\begin{equation*}
    \pi_{1,x}^\top P_x \pi'_x\geq \pi_{1,x}^\top P_x \pi_{2,x}\geq \pi_{x}^\top P_x \pi_{2,x},
\end{equation*}
which is equivalent to
\begin{equation*}
    (\pi'_x)^\top P_x^\top \pi_{1,x} \geq \pi_{2,x}^\top P_x^\top \pi_{1,x}\geq \pi_{2,x}^\top P_x^\top \pi_x.
\end{equation*}
Note that 
\begin{equation}\label{eq:shifted_skew_symmetry}
    P_x+P_x^\top=J,
\end{equation}
where $J\in\R^{|\gY|\times|\gY|}$ is the matrix of all ones. This gives
\begin{equation*}
    -(\pi'_x)^\top P_x \pi_{1,x} \geq -\pi_{2,x}^\top P_x \pi_{1,x}\geq -\pi_{2,x}^\top P_x \pi_x,
\end{equation*}
i.e.,
\begin{equation*}
    (\pi'_x)^\top P_x \pi_{1,x} \leq \pi_{2,x}^\top P_x \pi_{1,x}\leq \pi_{2,x}^\top P_x \pi_x,\,\, \forall \pi,\pi'\in\triangle_\gY^\gX,\,\,\forall x\in \gX.
\end{equation*}
This implies $(\pi_2,\pi_1)$ is also a Nash equilibrium of \eqref{eq:minmax}. Then by the interchangeability of Nash equilibrium strategies for two-player zero-sum games~\citep{nash1950non}, $(\pi_1,\pi_1)$ and $(\pi_2,\pi_2)$ are both the Nash equilibria of \eqref{eq:minmax}, which indicates that $\pi_1,\pi_2$ are both the solutions of \eqref{eq:optimal_policy_regularized}.

Next we prove the case when $\beta>0$.
When $\beta>0$, due to the strong concavity-convexity of the  minmax problem~\eqref{eq:minmax}, there exists a unique Nash equilibrium $(\pi_1^\star,\pi_2^\star)$ of it. Moreover, it is straightforward to compute that $(\pi_1^\star,\pi_2^\star)$ satisfies the following relation:
\begin{equation}\label{eq:zero_sum_regularized_sol}
\forall x\in\gX:\quad
    \begin{cases}
        \pi_1^\star(\cdot|x)\propto \pi_{\text{ref}}(\cdot|x)\circ\exp\left(\frac{1}{\beta}P_x\pi_2^\star(\cdot|x)\right), \\
        \pi_2^\star(\cdot|x)\propto \pi_{\text{ref}}(\cdot|x)\circ\exp\left(-\frac{1}{\beta}P_x^\top\pi_1^\star(\cdot|x)\right),
    \end{cases}
\end{equation}
where we use $\circ$ to denote the element-wise product of two vectors.

Again, using \eqref{eq:shifted_skew_symmetry}, we have
\begin{equation*}
\forall x\in\gX:\quad
    \begin{cases}
        \pi_1^\star(\cdot|x)\propto \pi_{\text{ref}}(\cdot|x)\circ\exp\left(\frac{1}{\beta}P_x\pi_2^\star(\cdot|x)\right), \\
        \pi_2^\star(\cdot|x)\propto \pi_{\text{ref}}(\cdot|x)\circ\exp\left(\frac{1}{\beta}P_x\pi_1^\star(\cdot|x)\right),
    \end{cases}
\end{equation*}
which implies $(\pi_2^\star,\pi_1^\star)$ is also a Nash equilibrium of \eqref{eq:minmax}. By the uniqueness of the Nash equilibrium we immediately know that $\pi_1^\star=\pi_2^\star$. Letting $\pi^\star_\beta=\pi_1^\star=\pi_2^\star$, we have $\pi^\star_\beta$ satisfies~\eqref{eq:optimal_policy_regularized}. 

On the other hand, if $\pi^\star_\beta$ is the solution of \eqref{eq:optimal_policy_regularized}, then $(\pi^\star_\beta, \pi^\star_\beta)$ satisfies \eqref{eq:zero_sum_regularized_sol} and thus is a Nash equilibrium of \eqref{eq:minmax}. In addition, by the uniqueness of \eqref{eq:minmax}, we deduce that \eqref{eq:optimal_policy_regularized} has a unique solution.


\subsection{Proofs of Theorem~\ref{thm:approx_wind} and Theorem~\ref{thm:wind_approx_relation}}\label{sec_app:proof_approx_wind}

We merge Theorem~\ref{thm:approx_wind} and Theorem~\ref{thm:wind_approx_relation} into the following theorem (recall we define $\overline{P}_x$ in \eqref{eq:P_bar}):

\begin{thm}[solution to iterative BoN (formal)]\label{thm:approx_wind_formal}
    Let $\piref\in \text{relint}\left(\triangle_\gY^\gX\right)$ and $n\geq 2$ in Algorithm~\ref{alg:iter_BoN}. Then $\lim_{T\rightarrow\infty}\pi_T$ exists in the following two cases:
    \begin{itemize}
        \item (No-mixing)  $\alpha_{1}=1,\alpha_{2}=0$. In this case $\overline\pi^\star_0\coloneqq\lim_{T\rightarrow\infty}\pi_T$ is a solution to both \eqref{eq:log_fixed_point} and \eqref{eq:optimal_policy_regularized} with $\beta=0$.
        \item (Mixing) $\alpha_1=\frac{\eta}{(1+\beta\eta)(n-1)}$, $\alpha_2=\frac{n-1-\eta}{(1+\beta\eta)(n-1)}$ for any $\beta,\eta>0$. In this case $\overline\pi^\star_\beta\coloneqq\lim_{T\rightarrow\infty}\pi_T$ satisfies:
        \begin{align}\label{eq:solution_regularized}
            \limitibon&\in\arg\max_{\pi} \underset{x\sim\rho,\atop y\in\pi(\cdot|x)}{\E}\log \underset{y'\in\limitibon(\cdot|x)}{\E}\overline P_x(y\succeq y')- \beta\dist{\pi}{\piref}.
        \end{align}
        
       Moreover, if 
       \begin{align}\label{eq:beta_bd}
           \beta\leq \min_{x\in\gX,\atop y\in\gY\backslash\gY^\star(x)}\left\{\frac{\sum_{y^\star\in\gY^\star(x)}\piref(y^\star|x)}{4\max\left\{\log\frac{\piref(y|x)}{\underset{y^\star\in\gY^\star(x)}{\max}\piref(y^\star|x)},0\right\}}\right\},
       \end{align}
       then for all $x\in\gX$, we have
       \begin{align}\label{eq:diff_limit}
           \quad\norm{\overline\pi^\star_{\beta,x}-\pi^\star_{\beta,x}}_1\leq 4(|\gY|-|\gY^\star(x)|)
           e^{\frac{-\sum_{y^\star\in\gY^\star(x)} \piref(y^\star|x)}{4\beta}}\rightarrow 0\text{ as }\beta\rightarrow 0.
       \end{align}
    \end{itemize}
\end{thm}

\begin{rmk}\label{rmk_app:equiv}
    It's easy to see that $\limitibon$ is a solution to \eqref{eq:solution_regularized} (see also \eqref{eq:log_fixed_point}) if and only if $(\limitibon,\limitibon)$ is a Nash equilibrium of the log-win-rate game~\eqref{eq:log_game}.
\end{rmk}

Now we give the proof of~Theorem~\ref{thm:approx_wind_formal}.

\paragraph{Step 1: show convergence for the no-mixing case.} We first prove the convergence result for the no-mixing case. Note that for any policy $\pi$, $\pi^{(n)}$ has the expression
\begin{align}\label{eq:BoN_expression}
    \forall (x,y)\in\gX\times\gY:\quad \pi^{(n)}(y|x)&=\binom{n}{1}\pi(y|x)\PP_{y_i\sim \pi(\cdot|x)}\left(r(x,y)\geq\max_{1\leq i\leq n-1} r(x,y_i)\right)\notag\\
&= n \pi(y|x)\left(\overline{P}_x(y,:)\pi_x\right)^{n-1},
\end{align}
where $\overline{P}_x$ is defined in \eqref{eq:P_bar}. 

When $\alpha_1=1,\alpha_2=0$, Algorithm~\ref{alg:iter_BoN} can be simplified as
$$\forall t\in\NN:\quad\pi_{t+1}=\pi_t^{(n)}.$$
Then $\pi_T$ is equivalent to $\piref^{\left(n^T\right)}$ --- the best of-$n^T$ policy of $\piref$. For any $x$, define $\gY^\star(x)$ as the set of responses that maximize the reward function $r(x,\cdot)$, i.e.,
$$\gY^\star(x)\coloneqq\argmax_{y\in\gY} r(x,y).$$
Then for any $y\in\gY^\star(x)$ and $y'\in\gY$, we have $\overline{P}_x(y,y')=1$. In addition, 
\begin{align*}
    \forall y\in\gY^\star(x):\quad \overline{P}_x(y,:) \pi_{\text{ref},x}&=1,\\
    \forall y\in\gY\setminus\gY^\star(x):\quad \overline{P}_x(y,:) \pi_{\text{ref},x}& < 1,
\end{align*}
where the latter holds since $\pi_{\text{ref}}\in\text{relint}\left(\triangle_\gY^\gX\right)$. 
By the BoN expression~\eqref{eq:BoN_expression}, we deduce that 
\begin{align}\label{eq:iterBoN0}
    \lim_{T\rightarrow\infty}\pi_T(y|x)=\begin{cases}
        \frac{\piref(y|x)}{\sum_{y\in\gY^\star(x)}\piref(y|x)}, & \textnormal{if }y\in\gY^\star(x),\\
        0, & \textnormal{otherwise}.
    \end{cases}
\end{align}

We let $\overline\pi^\star_0\coloneqq\lim_{T\rightarrow\infty}\pi_T$. Now we show that $(\overline\pi^\star_0,\overline\pi^\star_0)$ is a Nash equilibrium of \eqref{eq:minmax} when $\beta=0$, which implies $\overline\pi^\star_0$ is a solution to \eqref{eq:optimal_policy_regularized} when $\beta=0$.
Note that for any $(x,y)\in\gX\times\gY$, we have
\begin{align*}
    (\overline\pi^\star_{0,x})^\top P_x(:,y)=\begin{cases}
    \frac{1}{2},\quad & \textnormal{if }y\in\gY^\star(x),\\
    1,\quad & \textnormal{otherwise},
    \end{cases}
\end{align*}
which gives 
\begin{align}\label{eq:minmax_1}
    \forall \pi,x:\quad (\overline\pi^\star_{0,x})^\top P_x\pi_x\geq \frac{1}{2}= (\overline\pi^\star_{0,x})^\top P_x\overline\pi^\star_{0,x}.
\end{align}

On the other hand, for any $(x,y)\in\gX\times\gY$, we have
\begin{align*}
    P_x(y,:)\overline\pi^\star_{0,x}=\begin{cases}
    \frac{1}{2},\quad & \textnormal{if }y\in\gY^\star(x),\\
    0,\quad & \textnormal{otherwise},
    \end{cases}
\end{align*}
which implies
\begin{align}\label{eq:minmax_2}
    \forall \pi,x:\quad (\pi_x)^\top P_x\overline\pi^\star_{0,x}\leq \frac{1}{2}= (\overline\pi^\star_{0,x})^\top P_x\overline\pi^\star_{0,x}.
\end{align}
Together, \eqref{eq:minmax_1} and \eqref{eq:minmax_2} indicate that $(\overline\pi^\star_0,\overline\pi^\star_0)$ is a Nash equilibrium of \eqref{eq:minmax} when $\beta=0$, which also indicates that $\overline\pi^\star_0$ is a solution to \eqref{eq:optimal_policy_regularized}. It's straightforward to verify with \eqref{eq:iterBoN0} that $\overline\pi^\star_0$ is also a solution to \eqref{eq:log_fixed_point}.


\paragraph{Step 2: show convergence of the mixing case.} 
Recall that in the mixing case we set $\alpha_1=\frac{\eta}{(1+\beta\eta)(n-1)},\alpha_2=\frac{n-1-\eta}{(1+\beta\eta)(n-1)}$. We take logarithm on both sides of the iteration in line 5 of Algorithm~\ref{alg:iter_BoN} and
unroll it as follows:
\begin{align}\label{eq:unroll}
    \log \pi_{t+1} (y|x)&= \alpha_1 \log \widetilde{\pi}_{t} (y|x) +\alpha_2 \log\pi_t (y|x) +(1-\alpha_1-\alpha_2)\log \piref(y|x)+c_x\notag\\
    &= (\alpha_1+\alpha_2) \log\pi_t +(n-1)\alpha_1\log\left(\overline{P}_x(y,:)\pi_{t,x}\right) + (1-\alpha_1-\alpha_2)\log \piref(y|x)+c_x'\notag\\
&= (\alpha_1+\alpha_2)^{t+1} \log \pi_0(y|x) + (n-1)\alpha_1\sum_{i=0}^{t}(\alpha_1+\alpha_2)^i\log \left(\overline{P}_x(y,:)\pi_{t-i,x}\right) \notag\\
&\quad + (1-(\alpha_1+\alpha_2)^{t+1})\log \piref(y|x)+c_x''\notag\\
&= \log \piref(y|x) + (n-1)\alpha_1\sum_{i=0}^{t}(\alpha_1+\alpha_2)^i\log \left(\overline{P}_x(y,:)\pi_{t-i,x}\right)+c_x''\notag\\
&= \log \piref(y|x) + \frac{\eta}{1+\eta\beta}\sum_{i=0}^{t}\left(\frac{1}{1+\beta\eta}\right)^i\log \left(\overline{P}_x(y,:)\pi_{t-i,x}\right)+c_x'',
\end{align}
where $c_x,c_x',c_x''$ are constants that depend on $x$, and the second equality makes use of \eqref{eq:BoN_expression}, and the last equality follows from our choice of $\alpha_1,\alpha_2$.

Note that for all $\pi\in\triangle_\gY^\gX$, we have
\begin{align*}
    \forall y\in\gY^\star(x):\quad \overline{P}_x(y,:)\pi_{x}&=1,\\
    \forall y\in\gY\setminus\gY^\star(x):\quad \overline{P}_x(y,:)\pi_{x}&\leq 1.
\end{align*}
For each $x\in\gX$, we let $y_1(x)\in\gY^\star(x)$ such that $\piref(y_1|x)=\max_{y\in\gY^\star(x)}\piref(y|x)$. For notation simplicity, when it does not cause confusion, we simply write $y_1(x)$ as $y_1$. \eqref{eq:unroll} indicates that for all $y\in\gY$, we have
\begin{align}
    \log \left(\frac{\pi_{t+1}(y_1|x)}{\pi_{t+1}(y|x)}\right) = \log \left(\frac{\piref(y_1|x)}{\piref(y|x)}\right),\,\,\text{if }y\in\gY^\star(x),\label{eq:ratio_1}\\
    \log \left(\frac{\pi_{t+1}(y_1|x)}{\pi_{t+1}(y|x)}\right) = \log \left(\frac{\piref(y_1|x)}{\piref(y|x)}\right)+\frac{\eta}{1+\eta\beta}\sum_{i=0}^{t}\left(\frac{1}{1+\beta\eta}\right)^i\log \left(\frac{1}{\overline{P}_x(y,:)\pi_{t-i,x}}\right),\,\,\text{if }y\notin\gY^\star(x).\label{eq:ratio_2}
\end{align}
Especially, \eqref{eq:ratio_2} indicates that the ratio $\frac{\pi_{t+1}(y|x)}{\pi_{t+1}(y_1|x)}$ is decreasing with $t$ for all $y\notin\gY^\star(x)$. Since it has a lower bound 0, we have that the ratio $\frac{\pi_{t+1}(y|x)}{\pi_{t+1}(y_1|x)}$ converges as $t\rightarrow\infty$ for all $y\notin\gY^\star(x)$. Therefore, \eqref{eq:ratio_1} together with \eqref{eq:ratio_2} implies that $\limitibon\coloneqq\lim_{t\rightarrow\infty}\pi_{t+1}$ exists. To see that $\limitibon$ is a solution to \eqref{eq:solution_regularized}, we make use of the following lemma. 

\begin{lm}\label{lm:series_convergence}
    For any sequence $\{a_t\}_{t=0}^\infty$ in $\R$ where $a_t\leq 0$ for all $t$ and $a\coloneqq\lim_{t\rightarrow\infty}a_t$ exists ($a$ can be $-\infty$), for any $\alpha\in(0,1)$, we have
    \begin{equation}\label{eq:series_convergence}
        \lim_{t\rightarrow\infty}\sum_{i=0}^{t}\alpha^i a_{t-i}=\frac{a}{1-\alpha}.  
    \end{equation}
\end{lm}
\begin{proof}[Proof of Lemma~\ref{lm:series_convergence}]
    If $a=-\infty$, then $$\lim_{t\rightarrow\infty}\sum_{i=0}^{t}\alpha^i a_{t-i}\leq \lim_{t\rightarrow\infty}a_t =-\infty.$$
    If $a>-\infty$, we have
    $$\sum_{i=0}^t \alpha^i a_{t-i}=\sum_{i=0}^t \alpha^i a+\sum_{i=0}^t \alpha^i (\underbrace{a_{t-i}-a}_{e_{t-i}})=\frac{1-\alpha^{t+1}}{1-\alpha}a + \sum_{i=0}^t \alpha^i e_{t-i},$$
    thus we only need to verify that $\lim_{t\rightarrow\infty}\sum_{i=0}^{t}\alpha^i e_{t-i}=0$.

    For any $\epsilon>0$, there exists $N\in\NN$ such that 
    $$\forall t\geq N:\quad \left|\sum_{i=N}^t \alpha^i e_{t-i}\right|\leq \frac{\alpha^N}{1-\alpha} b\leq \epsilon/2,$$
    where $b=\max_{i\geq N}|e_i|$, and $b<\infty$ because $e_t$ converges to 0. 
    We fix $N$ and choose $T$ such that for all $t\geq T$, we have
    $$\sum_{i=0}^N \alpha^i a_{t-i}\leq \epsilon/2.$$ 
    Then for all $t\geq T$, we have
    \begin{align*}
        \left|\sum_{i=0}^{t}\alpha^i e_{t-i}\right|&\leq \left|\sum_{i=0}^{N}\alpha^i e_{t-i}\right|+\left|\sum_{i=N+1}^{t}\alpha^i e_{t-i}\right|\leq \epsilon/2+\epsilon/2=\epsilon.
    \end{align*}
    This completes the proof.
\end{proof}
Let $t\rightarrow\infty$ on both sides of \eqref{eq:unroll}, and by Lemma~\ref{lm:series_convergence}, we have
\begin{align}\label{eq:bon_limit}
    \limitibon(y|x)\propto\piref(y|x)\frac{1}{\beta}\log \left(\overline{P}_x(y,:)\overline{\pi}^\star_{\beta,x}\right).
\end{align}

Note that for any $x$, the strongly-concave problem
$$
\max_{\pi_x\in\gY}\log \left(\underset{y'\sim\overline\pi^\star_\beta(\cdot|x)}{\E}\overline{P}_x(y,y')\right) - \beta \textnormal{KL}(\pi_x||\piref(\cdot|x))
$$
has a unique solution  $\overline\pi^\star_{\beta,x}$. Therefore, $\limitibon$ is a solution to \eqref{eq:solution_regularized}. By Remark~\ref{rmk_app:equiv} we know that $(\limitibon,\limitibon)$ is a Nash equilibrium of the log-win-rate game~\eqref{eq:log_game}.

\paragraph{Step 3: bound the distance between $\pi^\star_\beta$ and $\limitibon$.} 
We let $\pi^{(0)}=\piref$ in Algorithm~\ref{alg:iter_PMDA} and unroll the iteration~\eqref{eq:update_mirror_analytical} similar to \eqref{eq:unroll}. We have
\begin{align}\label{eq:unroll_mirror}
    \log \pi^{(t+1)} (y|x)= \log \piref(y|x) + \frac{\eta}{1+\eta\beta}\sum_{i=0}^{t}\left(\frac{1}{1+\beta\eta}\right)^i  P_x(y,:)\pi_x^{(t-i)}+c_x''',
\end{align}
where $\pi^{(t)}$ is the policy at the $t$-th round of Algorithm~\ref{alg:iter_PMDA}. Furthermore, similar to \eqref{eq:ratio_1} and \eqref{eq:ratio_2}, we have
\begin{align}
    \log \left(\frac{\pi^{(t+1)}(y_1|x)}{\pi^{(t+1)}(y|x)}\right) = \log \left(\frac{\piref(y_1|x)}{\piref(y|x)}\right),\,\,\text{if }y\in\gY^\star(x),\label{eq:ratio_1_mirror}\\
    \log \left(\frac{\pi^{(t+1)}(y_1|x)}{\pi^{(t+1)}(y|x)}\right) = \log \left(\frac{\piref(y_1|x)}{\piref(y|x)}\right)+\frac{\eta}{1+\eta\beta}\sum_{i=0}^{t}\left(\frac{1}{1+\beta\eta}\right)^i\left(P_x(y_1,:)-P_x(y,:)\right)\pi_x^{(t-i)},\,\,\text{if }y\notin\gY^\star(x).\label{eq:ratio_2_mirror}
\end{align}
For any $\pi\in\triangle_\gY^\gX$, we have
\begin{align}\label{eq:increase_mirror}
    \forall y\neq\gY^\star(x):\quad \left(P_x(y_1,:)-P_x(y,:)\right)\pi_x\geq \frac{1}{2}\sum_{y\in\gY^\star(x)}\pi(y|x)\geq 0,
\end{align}
we know that $\log \left(\frac{\pi^{(t+1)}(y_1|x)}{\pi^{(t+1)}(y|x)}\right)$ is increasing with $t$ for all $y\in\gY\setminus\gY^\star(x)$, and thus $\pi^{(t)}(y|x)$ is decreasing with $t$ for all $y\in\gY\setminus\gY^\star(x)$. Moreover, by a similar argument as in Step 1, we have that $\lim_{t\rightarrow\infty}\pi^{(t)}$ exists and is the solution to \eqref{eq:optimal_policy_regularized} (even when $\eta>\beta$).
    
    Note that \eqref{eq:ratio_2_mirror} is equivalent to
\begin{align}\label{eq:mirror_limit}
    \log \left(\frac{\pi^{(t+1)}(y_1|x)}{\pi^{(t+1)}(y|x)}\right) &= \frac{\eta}{1+\eta\beta}\sum_{i=0}^{t}\left(\frac{1}{1+\beta\eta}\right)^i\bigg(\underbrace{\left(P_x(y_1,:)-P_x(y,:)\right)\pi_x^{(t-i)}+\beta\log \left(\frac{\piref(y_1|x)}{\piref(y|x)}\right)}_{\xi^{(t-i)}} \bigg)\notag\\
    & \quad+\left(\frac{1}{1+\beta\eta}\right)^{t+1}\log \left(\frac{\piref(y_1|x)}{\piref(y|x)}\right).
\end{align}
Also note that by \eqref{eq:increase_mirror} and the decreasing property of $\pi^{(t)}(y|x)$ for all $y\notin\gY^\star(x)$, we have
\begin{align*}
    \forall x\in\gX,\forall y\neq\gY^\star(x):\quad \left(P_x(y_1,:)-P_x(y,:)\right)\pi_x^{(t-i)}&\geq \frac{1}{2}\sum_{y\in\gY^\star(x)}\piref(y|x).
\end{align*}
From the above expression and our choice of $\beta$ we know that 
\begin{align}\label{eq:delta_mirror}
    \forall x\in\gX,\forall y\neq\gY^\star(x):\quad \xi^{(t-i)}\geq \frac{1}{4}\sum_{y\in\gY^\star(x)}\piref(y|x).
\end{align}
Then by \eqref{eq:mirror_limit} we know that
\begin{align*}
    \forall x\in\gX,\forall y\neq\gY^\star(x):\quad \log \left(\frac{\pi_\beta^\star(y_1|x)}{\pi_\beta^\star(y|x)}\right)\geq  \frac{1}{4\beta}\sum_{y\in\gY^\star(x)}\piref(y|x),
\end{align*}
which indicates that for all $y\in\gY\setminus\gY^\star(x)$,
\begin{align}
    \pi_\beta^\star(y|x)&\leq \pi_\beta^\star(y_1|x)\exp\left(-\frac{1}{4\beta}\sum_{y\in\gY^\star(x)}\piref(y|x)\right)\leq \exp\left(-\frac{1}{4\beta}\sum_{y\in\gY^\star(x)}\piref(y|x)\right),
\end{align}
which gives
$$
\sum_{y\in\gY\setminus\gY^\star(x)}\pi_\beta^\star(y|x)\leq (|\gY|-|\gY^\star(x)|)\exp\left(-\frac{1}{4\beta}\sum_{y\in\gY^\star(x)}\piref(y|x)\right).
$$
Combining the above relation with \eqref{eq:ratio_1_mirror}, we obtain
$$
\forall y\in\gY^\star(x):\quad \pi_\beta^\star(y|x)\geq \frac{\piref(y|x)}{\sum_{y\in\gY^\star(x)}\piref(y|x)}\cdot \left(1-(|\gY|-|\gY^\star(x)|)\exp\left(-\frac{1}{4\beta}\sum_{y\in\gY^\star(x)}\piref(y|x)\right)\right),
$$
and
$$\forall y\in\gY^\star(x):\quad \pi_\beta^\star(y|x)\leq \frac{\piref(y|x)}{\sum_{y\in\gY^\star(x)}\piref(y|x)}.$$

Recall that we write the expression of $\overline\pi^\star_0$ in \eqref{eq:iterBoN0}. Therefore, we have
\begin{align}\label{eq:bound_mirror0}
    \forall x\in\gX:\quad\norm{\pi^\star_{\beta,x}-\overline\pi^\star_{0,x}}_1\leq 2(|\gY|-|\gY^\star(x)|)\exp\left(-\frac{1}{4\beta}\sum_{y\in\gY^\star(x)}\piref(y|x)\right).
\end{align}

For the iteration in Algorithm~\ref{alg:iter_BoN}, similar to \eqref{eq:mirror_limit} we have 
\begin{align}
    \log \left(\frac{\pi^{(t+1)}(y_1|x)}{\pi^{(t+1)}(y|x)}\right) &= \frac{\eta}{1+\eta\beta}\sum_{i=0}^{t}\left(\frac{1}{1+\beta\eta}\right)^i\bigg(\underbrace{\log \left(\frac{1}{\overline{P}_x(y,:)\pi_{t-i,x}}\right)+\beta\log \left(\frac{\piref(y_1|x)}{\piref(y|x)}\right)}_{\delta^{(t-i)}} \bigg)\notag\\
    & \quad+\left(\frac{1}{1+\beta\eta}\right)^{t+1}\log \left(\frac{\piref(y_1|x)}{\piref(y|x)}\right).
\end{align}
Note that for all $y\in\gY\setminus\gY^\star(x)$, we have
$$
\log \left(\frac{1}{\overline{P}_x(y,:)\pi_{t-i,x}}\right)\geq \log \left(\frac{1}{1-\sum_{y\in\gY^\star(x)}\pi_{t-i}(y|x)}\right)\geq \log \left(\frac{1}{1-\sum_{y\in\gY^\star(x)}\piref(y|x)}\right)\geq \sum_{y\in\gY^\star(x)}\piref(y|x),
$$ 
where in the second inequality we use the fact that $\pi_t(y|x)$ is decreasing with $t$ for all $y\notin\gY^\star(x)$. Then by a similar argument as in \eqref{eq:delta_mirror}, we have
\begin{align}\label{eq:delta_ibon}
    \forall x\in\gX,\forall y\neq\gY^\star(x):\quad \delta^{(t-i)}\geq \frac{3}{4}\sum_{y\in\gY^\star(x)}\piref(y|x).
\end{align} 
Therefore, analogous to \eqref{eq:bound_mirror0}, we have
\begin{align}\label{eq:bound_ibon0}
    \forall x\in\gX:\quad\norm{\pi^\star_{\beta,x}-\overline\pi^\star_{\beta,x}}_1\leq 2(|\gY|-|\gY^\star(x)|)\exp\left(-\frac{3}{4\beta}\sum_{y\in\gY^\star(x)}\piref(y|x)\right).
\end{align}
Combining \eqref{eq:bound_mirror0} and \eqref{eq:bound_ibon0}, we have
\begin{align*}
\norm{\overline\pi^\star_{\beta,x}-\pi^\star_{\beta,x}}_1&\leq \norm{\overline\pi^\star_{\beta,x}-\overline\pi^\star_{0,x}}_1+\norm{\pi^\star_{\beta,x}-\overline\pi^\star_{0,x}}_1\\
&\leq 2(|\gY|-|\gY^\star(x)|)\left(\exp\left(-\frac{3}{4\beta}\sum_{y\in\gY^\star(x)}\piref(y|x)\right)+\exp\left(-\frac{1}{4\beta}\sum_{y\in\gY^\star(x)}\piref(y|x)\right)\right),
\end{align*}
from which we can see that \eqref{eq:diff_limit} holds.

\subsection{Proof of Theorem~\ref{thm:rate}}\label{sec_app:proof_thm_rate}

To start with, we reformulate problem~\eqref{eq:optimal_policy_regularized} as a monotone variational inequality (VI) problem.
We first define the operator $F_x:\triangle_\gY\rightarrow\R^{|\gY|}$ for all $x\in\gX$ as
\begin{equation}\label{eq:F_x}
    F_x(\pi_x)\coloneqq -P_x\pi_x-\beta\log\pi_{\text{ref},x},\,\,\forall \pi_x\in \triangle_\gY.
\end{equation}

We also let 
\begin{equation}\label{eq:negative_entropy}
  h:\triangle_\gY\rightarrow\R,\,\, h(p)\coloneqq\sum_i p_i\log p_i
\end{equation}
denote the negative entropy.
The following lemma gives the VI form of \WIND.
\begin{lm}\label{lm:VI}
    Assume $\beta>0$ and $\pi^{(0)},\piref\in\text{relint}(\triangle^\gX_\gY)$. Then \eqref{eq:optimal_policy_regularized} is equivalent to the following monotone VI problem:
    \begin{equation}\label{eq:MVI}
      \underset{x\sim\rho}{\E}\left[\innerprod{F_x(\pi^\star_{\beta,x})+\beta\nabla h(\pi_{\beta,x}^\star), \pi_x-\pi^\star_{\beta,x}}\right]\geq 0,\,\,\forall \pi\in\Delta_\gY^\gX,
    \end{equation}
    where for all $x\in\gX$, $F_x$ is monotone and 1-Lipschitz continuous w.r.t. the $\ell_1$-norm. 
\end{lm}
\begin{proof}[Proof of Lemma~\ref{lm:VI}]
    By the proof of Proposition~\ref{prop:existence} we know that when $\beta>0$ and $\piref\in\text{relint}(\triangle^\gX_\gY)$, we have $\pi^\star_\beta\in\text{relint}(\triangle^\gX_\gY)$. By the optimality condition, $\pi^\star_\beta$ satisfies \eqref{eq:optimal_policy_regularized} if and only if
\begin{align}\label{eq:optimality_condition}
    \innerprod{\nabla f^\star(\pi^\star_\beta), \pi-\pi^\star_\beta}\geq 0,\,\,\forall \pi,
\end{align}
where
$$f^\star(\pi)\coloneqq\underset{x\sim\rho}{\E}\big[\innerprod{\pi_x,\, -P_x\pi^\star_{\beta,x}-\beta\log\pi_{\text{ref},x}+\beta\log\pi_x}\big].$$
By \eqref{eq:F_x} and \eqref{eq:negative_entropy}, we have \eqref{eq:optimality_condition} equivalent to \eqref{eq:MVI}.
To see the monotonicity of $F_x$, we have
\begin{align}\label{eq:monotonicity}
   \innerprod{F_x(\pi_x)-F_x(\pi_x'),\pi_x-\pi_x'} 
    &=(\pi_x-\pi_x')^\top P_x(\pi_x-\pi_x')\notag\\
    &=(\pi_x-\pi_x')^\top \frac{1}{2}(P_x+P_x^\top)(\pi_x-\pi_x')\notag\\
    &=(\pi_x-\pi_x')^\top\frac{1}{2} J (\pi_x-\pi_x')=0,
\end{align}
where $J\in\R^{|\gY|\times|\gY|}$ is the matrix of all ones.

Furthermore, we have
\begin{align}\label{eq:Lip}
    \forall x\in\gX, \forall p,q\in\triangle_\gY:\quad\norm{F_x(p)-F_x(q)}_\infty\leq \norm{P_x(p-q)}_\infty\leq\norm{p-q}_1,
\end{align}
where the second inequality follows from the fact that each entry of $P_x$ only take its value in $\{0,1/2, 1\}$. We finish up by noting \eqref{eq:Lip} indicates that $F_x$ is 1-Lipschitz with respect to $\ell_1$-norm. 
\end{proof}

We use the following \textit{proximal mirror descent ascent} rule~\citep{sokota2023unifiedapproachreinforcementlearning,pattathil2023symmetric} to solve the monotone VI problem~\eqref{eq:MVI}:
\begin{align}\label{eq:update_mirror}
\pi^{(t+1)}=\arg\min_{\pi}\underset{x\sim\rho}{\E}\bigg[\innerprod{F_x(\pi^{(t)}_x), \pi_x}+\beta h(\pi_x)+\frac{1}{\eta} B_h(\pi_x,\pi^{(t)}_x)\bigg],
\end{align}
where $\eta>0$ is the learning rate, and the Bregman distance $B_h:\triangle_\gY\times\triangle_\gY\rightarrow\R_+$ is generated from the negative entropy $h$: 
\begin{equation*}
  B_h(p,q)\coloneqq h (p)-h(q)-\innerprod{\nabla h(q),p-q}=\text{KL}(p||q).
\end{equation*}

It's straightforward to verify that the analytical solution of \eqref{eq:update_mirror} is \eqref{eq:update_mirror_analytical} in Algorithm~\ref{alg:iter_PMDA}.
Note that the negative entropy $h$ (c.f.~\eqref{eq:negative_entropy}) is 1-strongly convex on $\triangle_\gY$ with respect to the $\ell_1$-norm~\citep[Example 5.27]{beck2017first}. Furthermore, Lemma~\ref{lm:VI} shows $F_x$ is 1-Lipschitz with respect to $\ell_1$-norm. 
With these facts, the theorem follows directly from \citet[Theorem 3.4]{sokota2023unifiedapproachreinforcementlearning}:
\begin{align*}
        \forall x\in \gX:\quad \textnormal{KL}(\pi^\star_\beta(\cdot|x)||\pi^{(t)}(\cdot|x))\leq\left(\frac{1}{1+\eta\beta}\right)^t \textnormal{KL}(\pi^\star_\beta(\cdot|x)||\pi^{(0)}(\cdot|x)).
\end{align*}
The relation \eqref{eq:rate} can be deduced from the above relation by taking the expectation over $x\sim\rho$ on both sides.


\subsection{Proof of Lemma~\ref{lm:conditional_mean}}\label{sec_app:proof_lm_conditional_mean}

To start with, we have
\begin{align}\label{eq:square}
    &\quad\,\,\E_{u,v}\left[(v-g(u))^2\right]\notag\\
    &=\E_{u,v}\left[\left((v-\E_v(v|u))+(\E_v(v|u)-g(u))\right)^2\right]\notag\\
    &=\E_{u,v}\left[(v-\E_v(v|u))^2\right]+2\E_{u,v}\left[(v-\E_v(v|u))(\E_v(v|u)-g(u))\right]+\E_{u,v}\left[(\E_v(v|u)-g(u))^2\right].
\end{align}
We use $F(u)$, $F(u,v)$ and $F(v|u)$ to denote the distribution of $u$, the joint distribution of $u,v$ and the distribution of $v$ conditioned on $u$, resp. Then the cross term
\begin{align}\label{eq:cross}
    \E_{u,v}\left[(v-\E_v(v|u))(\E_v(v|u)-g(u))\right]&=\int_{(u,v)}(v-\E_v(v|u))(\E_v(v|u)-g(u)) d F(u,v)\notag\\
    &=\int_u \left(\int_v v-\E_v(v|u) d F(v|u) \right) (\E_v(v|u)-g(u)) d F(u)\notag\\
    &=0,
\end{align}
where the last relation follows from the fact that 
$$\int_v v-\E_v(v|u) d F(v|u)= \E_v(v|u)-\E_v(v|u)=0. $$

Combining \eqref{eq:cross} and \eqref{eq:square}, we have that
\begin{align*}
    \E_{u,v}\left[(v-g(u))^2\right] = \E_{u,v}\left[(v-\E_v(v|u))^2\right] + \E_{u}\left[(\E_v(v|u)-g(u))^2\right]\geq \E_{u,v}\left[(v-\E_v(v|u))^2\right],
\end{align*}
and the equality holds if and only if $g(u)=\E_v(v|u)$ almost everywhere on the support set of $F(u)$.


\subsection{Proof of Theorem~\ref{thm:complexity}}\label{sec_app:proof_complexity}

We first introduce the three-point property of the Bregman divergence~\citep[Proposition~D.1]{sokota2023unifiedapproachreinforcementlearning}, \citep[Proposition~2.3]{bauschke2003bregman}:
\begin{lm}[three-point property of the Bregman divergence]\label{lm:three_point}
    Let $\psi:\triangle_\gY\rightarrow \R$ 
    be a function that's differentiable on $int(\triangle_\gY)$. Let $p,q\in\triangle_\gY$ and $r,s\in int(\triangle_\gY)$. Then the following equality holds:
    \begin{align}
        B_\psi(r,s) + B_\psi(s,r) &= \innerprod{\nabla \psi(r)-\nabla \psi(s),r-s}.\label{eq:three_point_1}\\
        B_\psi(p,r) &= B_\psi(p,s)+B_\psi(s,r)+\innerprod{\nabla \psi(s)-\nabla \psi(r),p-s}.\label{eq:three_point_2}\\
        B_\psi(p,s) + B_\psi(q,r) &= B_\psi(p,r) + B_\psi(q,s)+\innerprod{\nabla \psi(r)-\nabla \psi(s),p-q}.\label{eq:three_point_3}
    \end{align}
\end{lm}

To start with, we rewrite the update rule in Algorithm~\ref{alg:practice}
\begin{align}\label{eq:regression}
    \theta_{t+1}\leftarrow \argmin_{\theta\in\Theta} \frac{1}{M}\sum_{i=1}^M\big(\varphi_t(x_i^{(t)},y_i^{(t)},{y'}^{(t)}_{i})-\phi_\theta(y_i^{(t)}|x_i^{(t)})\big)^2.
\end{align}
to a similar form as the update rule~\eqref{eq:update_mirror}.


\paragraph{Step 1: reformulate the update rule~\eqref{eq:regression}.} 
We let $\delta_{S}^{(t)},\delta_{P}^{(t)}\in\R^{|\gX|\times|\gY|}$ denote the statistical error and the model approximation error at the $t$-th round, respectively:
\begin{align}
   \forall (x,y)\in\gX\times\gY:\quad \delta_{S}^{(t)}(x,y)&\coloneqq \phi_{\theta_{t+1}}(y|x)-\phi_{\theta_{t+1}^\star}(y|x),\label{eq:statistical_error}\\
   \delta_{P}^{(t)}(x,y)&\coloneqq \E_{y'\sim\pi_{\theta_t}(\cdot|x)}\widehat P_x(y,y')-\E_{y'\sim\pi_{\theta_t}(\cdot|x)} P_x(y,y').\label{eq:model_error_bound}
\end{align}
We write $\delta_{S,x}^{(t)}\in\R^{|\gX|},\delta_{P,x}^{(t)}\in\R^{|\gX|}$ as the shorthand of $\left(\delta_{S}^{(t)}(x,y)\right)_{y\in\gY},\left(\delta_{P}^{(t)}(x,y)\right)_{y\in\gY}$, respectively.

The above expression~\eqref{eq:statistical_error} combined with \eqref{eq:expressive} gives
\begin{align}\label{eq:update_mirror_parameter}
    \pi_{\theta_{t+1}}(y|x)&\propto(\pi_{\theta_{t}}(y|x))^{\frac{1}{1+\beta\eta}}(\piref(y|x))^{\frac{\beta\eta}{1+\beta\eta}}\exp\left(\frac{\eta}{1+\beta\eta}\left(\E_{y'\sim\pi_{\theta_t}(\cdot|x)}\widehat P_x(y,y')+\frac{1+\beta\eta}{\eta}\delta_S^{(t)}(x,y)\right)\right).
\end{align}
For notation simplicity, we let $\Pi\coloneqq\triangle^{\gX}_{\gY}$ denote the whole policy space. 
Note that the above relation is equivalent to 
\begin{equation}\label{eq:update_reform}
    \pi_{\theta_{t+1}}=\arg\min_{\pi\in\Pi}\EE_{x\sim\rho}\left[\innerprod{\widehat F^{(t)}_x, \pi_x}+\beta h(\pi_x)+\frac{1}{\eta} B_h(\pi_x,\pi_{\theta_{t},x})\right],
\end{equation}
where $\widehat F^{(t)}_{x}\in\R^{|\gY|}$ is defined as
\begin{equation}\label{eq:hat_F}
    \forall x\in\gX:\quad
    \widehat F^{(t)}_{x}\coloneqq \underbrace{- P_x\pi_{\theta_t,x}-\beta\log\pi_{\text{ref},x}}_{= F_{x}(\pi_{\theta_t,x}) \text{ by \eqref{eq:F_x}}}-\frac{1+\beta\eta}{\eta}\delta_{S,x}^{(t)}-\delta_{P,x}^{(t)},
\end{equation}
which could be seen as an approximation of $F_{x}(\pi_{\theta_t,x})$. We let $\delta^{(t)}\in\R^{|\gX||\gY|}$ denote 
\begin{equation}
    \forall x\in\gX:\quad \delta^{(t)}_x\coloneqq\delta^{(t)}(x,\cdot)\coloneqq\widehat{F}^{(t)}_x-F_x(\pi_{\theta_t,x})=-\frac{1+\beta\eta}{\eta}\delta_{S,x}^{(t)}-\delta_{P,x}^{(t)}.\label{eq:delta}
\end{equation}


The next step is to bound the distance between $\pi_{\theta_t}$ and $\pi^\star_\beta$ utilizing the reformulated update rule~\eqref{eq:update_reform}.
This part of our proof is inspired by~\citet[Theorem 3.4]{sokota2023unifiedapproachreinforcementlearning}.



\paragraph{Step 2: bound $\dist{\pi^\star_\beta}{\pi_{\theta_t}}$.} By the first-order optimality condition we know that \eqref{eq:update_reform} is equivalent to
\begin{equation}\label{eq:theta_t+1_grad}
    \innerprod{\widehat{F}^{(t)}_x+\beta\nabla h(\pi_{\theta_{t+1},x})+\frac{1}{\eta}(\nabla h(\pi_{\theta_{t+1},x})-\nabla h(\pi_{\theta_{t},x})),\pi_x-\pi_{\theta_{t+1},x}}\geq 0,\,\,\forall \pi\in\Pi,\,\,\forall x\in\gX,
\end{equation}

Reorganizing the terms in \eqref{eq:theta_t+1_grad}, we have
\begin{align}\label{eq:theta_t+1_grad_2}
    \innerprod{\widehat{F}^{(t)}_x+\beta\nabla h(\pi_{\theta_{t+1},x}),\pi_x-\pi_{\theta_{t+1},x}}&\geq \frac{1}{\eta}\innerprod{\nabla h(\pi_{\theta_{t},x})-\nabla h(\pi_{\theta_{t+1},x}),\pi_x-\pi_{\theta_{t+1},x}}\notag\\
&\overset{\eqref{eq:three_point_2}}{=}\frac{1}{\eta}\left(-B_h(\pi_x,\pi_{\theta_t,x})+B_h(\pi_x,\pi_{\theta_{t+1},x})+B_h(\pi_{\theta_{t+1},x},\pi_{\theta_t,x})\right).
\end{align}
Let $\pi=\pi^\star_\beta$ in \eqref{eq:theta_t+1_grad_2} and reorganize the terms, we have
\begin{align}\label{eq:p1}
    &\quad\,\, B_h(\pi^\star_{\beta,x},\pi_{\theta_{t+1},x})\notag\\
    &\leq B_h(\pi^\star_{\beta,x},\pi_{\theta_{t},x})-B_h(\pi_{\theta_{t+1},x},\pi_{\theta_{t},x})+
    \eta\innerprod{\widehat{F}^{(t)}_x+\beta\nabla h(\pi_{\theta_{t+1},x}),\pi^\star_{\beta,x}-\pi_{\theta_{t+1},x}}\notag\\
    &=B_h(\pi^\star_{\beta,x},\pi_{\theta_{t},x})-B_h(\pi_{\theta_{t+1},x},\pi_{\theta_{t},x})+
    \eta\innerprod{F_x(\pi_{\theta_{t},x})+\beta\nabla h(\pi_{\theta_{t+1},x}),\pi^\star_{\beta,x}-\pi_{\theta_{t+1},x}} \notag\\
    &\quad+ \eta\innerprod{\delta^{(t)}_x,\pi^\star_{\beta,x}-\pi_{\theta_{t+1},x}}\notag\\
    &= B_h(\pi^\star_{\beta,x},\pi_{\theta_{t},x})-B_h(\pi_{\theta_{t+1},x},\pi_{\theta_{t},x})+
    \eta\innerprod{F_x(\pi_{\theta_{t},x})-F_x(\pi_{\theta_{t+1},x}),\pi^\star_{\beta,x}-\pi_{\theta_{t+1},x}} \notag\\
    &\quad + \eta\innerprod{F_x(\pi_{\theta_{t+1},x})+\beta\nabla h(\pi_{\theta_{t+1},x}),\pi^\star_{\beta,x}-\pi_{\theta_{t+1},x}}+ \eta\innerprod{\delta^{(t)}_x,\pi^\star_{\beta,x}-\pi_{\theta_{t+1},x}},\,\,\forall x\in\gX,\,\,\forall \pi\in\Pi.
\end{align}

Note that for any $\pi\in\Pi$ and $x\in\gX$, we have
\begin{align}\label{eq:p2}
    \innerprod{F_x(\pi_x)+\beta\nabla h(\pi_x),\pi^\star_{\beta,x}-\pi_x}
    &=\underbrace{\innerprod{F_x(\pi_x)-F_x(\pi^\star_{\beta,x}),\pi^\star_{\beta,x}-\pi_x}}_{=0 \text{ by~\eqref{eq:monotonicity}}}+\beta\innerprod{\nabla h(\pi_x)-\nabla h(\pi^\star_{\beta,x}),\pi^\star_{\beta,x}-\pi_x}\notag\\
    &\quad+\underbrace{\innerprod{F_x(\pi^\star_{\beta,x})+\beta \nabla h(\pi^\star_{\beta,x}), \pi^\star_{\beta,x}-\pi_x}}_{\leq 0 \text{ by~\eqref{eq:MVI}}}\notag\\
    &\leq \beta\innerprod{\nabla h(\pi_x)-\nabla h(\pi^\star_{\beta,x}),\pi^\star_{\beta,x}-\pi_x}\notag\\
    &\overset{\eqref{eq:three_point_1}}{=}-\beta \left(B_h(\pi_x,\pi^\star_{\beta,x})+ B_h(\pi^\star_{\beta,x},\pi_x)\right).
\end{align}

Combining the above two expressions \eqref{eq:p1} and \eqref{eq:p2}, we have
\begin{align}\label{eq:p3}
    B_h(\pi^\star_{\beta,x},\pi_{\theta_{t+1},x})&\leq  B_h(\pi^\star_{\beta,x},\pi_{\theta_{t},x})-B_h(\pi_{\theta_{t+1},x},\pi_{\theta_{t},x})+
    \eta\innerprod{F_x(\pi_{\theta_{t},x})-F_x(\pi_{\theta_{t+1},x}),\pi^\star_{\beta,x}-\pi_{\theta_{t+1},x}} \notag\\
    &\quad -\beta\eta \left(B_h(\pi_{\theta_{t+1},x},\pi^\star_{\beta,x})+ B_h(\pi^\star_{\beta,x},\pi_{\theta_{t+1},x})\right)+ \eta\innerprod{\delta^{(t)}_x,\pi^\star_{\beta,x}-\pi_{\theta_{t+1},x}}\notag\\
    &\overset{\eqref{eq:Lipschitz}}{\leq}  B_h(\pi^\star_{\beta,x},\pi_{\theta_{t},x})-B_h(\pi_{\theta_{t+1},x},\pi_{\theta_{t},x})+
    \eta \norm{\pi_{\theta_{t},x}-\pi_{\theta_{t+1},x}}_1 \norm{\pi^\star_{\beta,x}-\pi_{\theta_{t+1},x}}_1 \notag\\
    &\quad -\beta\eta \left(B_h(\pi_{\theta_{t+1},x},\pi^\star_{\beta,x})+ B_h(\pi^\star_{\beta,x},\pi_{\theta_{t+1},x})\right)+ \eta\innerprod{\delta^{(t)}_x,\pi^\star_{\beta,x}-\pi_{\theta_{t+1},x}}\notag\\\
    &\leq B_h(\pi^\star_{\beta,x},\pi_{\theta_{t},x})\underbrace{-B_h(\pi_{\theta_{t+1},x},\pi_{\theta_{t},x})+
    \frac{1}{2}\norm{\pi_{\theta_{t},x}-\pi_{\theta_{t+1},x}}_1^2}_{\leq 0} + \underbrace{\frac{\eta^2}{2}\norm{\pi^\star_{\beta,x}-\pi_{\theta_{t+1},x}}_1^2}_{\leq \eta^2 B_h(\pi_{\theta_{t+1},x},\pi^\star_{\beta,x})} \notag\\
    &\quad -\beta\eta \left(B_h(\pi_{\theta_{t+1},x},\pi^\star_{\beta,x})+ B_h(\pi^\star_{\beta,x},\pi_{\theta_{t+1},x})\right)+ \eta\innerprod{\delta^{(t)}_x,\pi^\star_{\beta,x}-\pi_{\theta_{t+1},x}}\notag\\
    &\leq B_h(\pi^\star_{\beta,x},\pi_{\theta_{t},x})+\eta\underbrace{(\eta-\beta)}_{\leq 0}B_h(\pi_{\theta_{t+1},x},\pi^\star_{\beta,x})-\beta\eta B_h(\pi^\star_{\beta,x},\pi_{\theta_{t+1},x})+ \eta\innerprod{\delta^{(t)}_x,\pi^\star_{\beta,x}-\pi_{\theta_{t+1},x}} \notag\\
    &\leq B_h(\pi^\star_{\beta,x},\pi_{\theta_{t},x})-\beta\eta B_h(\pi^\star_{\beta,x},\pi_{\theta_{t+1},x})+\eta\innerprod{\delta^{(t)}_x,\pi^\star_{\beta,x}-\pi_{\theta_{t+1},x}},
\end{align}
where, in the third relation we use the 1-strong convexity of $h$ w.r.t. the $l_1$-norm (see the proof of Theorem~\ref{thm:rate}) to obtain that
\begin{align*}
    &\quad-B_h(\pi_{\theta_{t+1},x},\pi_{\theta_{t},x})+
    \frac{1}{2}\norm{\pi_{\theta_{t},x}-\pi_{\theta_{t+1},x}}_1^2 \\
    &= -\left(h(\pi_{\theta_{t+1},x})-h(\pi_{\theta_{t},x})-\innerprod{\nabla h(\pi_{\theta_{t},x}),\pi_{\theta_{t+1},x}-\pi_{\theta_{t},x}}-\frac{1}{2}\norm{\pi_{\theta_{t},x}-\pi_{\theta_{t+1},x}}_1^2\right)\leq 0.
\end{align*}

Note that 
\begin{align}\label{eq:bound_error_pre}
    \innerprod{\delta^{(t)}_x,\pi^\star_{\beta,x}-\pi_{\theta_{t+1},x}}&\overset{\eqref{eq:delta}}{=} \innerprod{-\frac{1+\beta\eta}{\eta}\delta_{S,x}^{(t)}-\delta_{P,x}^{(t)},\pi^\star_{\beta,x}-\pi_{\theta_{t+1},x}}\notag\\
    & \leq \frac{1+\beta\eta}{\eta} \underbrace{\left|\innerprod{\delta^{(t)}_{S,x},\pi^\star_{\beta,x}}\right|}_{(i)} + \frac{1+\beta\eta}{\eta}\underbrace{ \left|\innerprod{\delta^{(t)}_{S,x},\pi_{\theta_{t+1},x}}\right|}_{(ii)} + \underbrace{\norm{\delta^{(t)}_{P,x}}_\infty\norm{\pi^\star_{\beta,x}-\pi_{\theta_{t+1},x}}_1}_{(iii)}.
\end{align}
To bound the error term, below we separately bound (i)-(iii).

    

To bound (i), we first unroll \eqref{eq:update_mirror_parameter} similar as in~\eqref{eq:unroll} and obtain
\begin{align}
    \log \pi_{\theta_{t+1}}(y|x)&=\frac{1}{1+\beta\eta}\log \pi_{\theta_{t}}(y|x) + \frac{\beta\eta}{1+\beta\eta}\piref(y|x)\notag\\
    &\quad+\frac{\eta}{1+\beta\eta}\left(\E_{y'\sim\pi_{\theta_t}(\cdot|x)}P_x(y,y')+\frac{1+\beta\eta}{\eta}\delta_S^{(t)}(x,y)+\delta_{P}^{(t)}(x,y)\right)\notag\\
    &=\log\piref(y|x) +\frac{\eta}{1+\eta\beta}\sum_{i=0}^t \left(\frac{1}{1+\beta\eta}\right)^i \bigg(\E_{y'\sim\pi_{\theta_{t-i}}(\cdot|x)}P_x(y,y')\notag\\
    &\qquad\qquad\qquad\qquad+\frac{1+\beta\eta}{\eta}\delta_S^{(t-i)}(x,y)+\delta_{P}^{(t-i)}(x,y)\bigg)+z_x
\end{align}
for some $z_x$ related to $x$. The above expression gives
\begin{align*}
    \log\left(\frac{\pi_{\theta_{t+1}}(y'|x)}{\pi_{\theta_{t+1}}(y|x)}\right)&=\log\left(\frac{\pi_{\textnormal{ref}}(y'|x)}{\pi_{\textnormal{ref}}(y|x)}\right)+\frac{\eta}{1+\eta\beta}\sum_{i=0}^t \left(\frac{1}{1+\beta\eta}\right)^i\bigg(\E_{y''\sim\pi_{\theta_{t-i}}(\cdot|x)}(P_x(y',y'')-P_x(y,y''))\notag\\
    &\quad+\frac{1+\beta\eta}{\eta}\delta_S^{(t-i)}(x,y')+\delta_{P}^{(t-i)}(x,y')-\frac{1+\beta\eta}{\eta}\delta_S^{(t-i)}(x,y)-\delta_{P}^{(t-i)}(x,y)\bigg),
\end{align*}
for any $y,y'\in\gY$.  This relation yields
\begin{align*}
    \log\left(\frac{\pi_{\theta_{t+1}}(y'|x)}{\pi_{\theta_{t+1}}(y|x)}\right)\leq \log\left(\frac{\pi_{\textnormal{ref}}(y'|x)}{\pi_{\textnormal{ref}}(y|x)}\right)+\frac{\eta}{1+\eta\beta}\sum_{i=0}^t \left(\frac{1}{1+\beta\eta}\right)^i \cdot 2\bigg(\delta_{P}+\frac{1+\beta\eta}{\eta}L_0 +1\bigg),
\end{align*}
where we use \eqref{eq:Lipschitz} and \eqref{eq:statistical_error}.
The above expression indicates
\begin{align*}
    \frac{\pi_{\theta_{t+1}}(y'|x)}{\pi_{\theta_{t+1}}(y|x)}\leq \frac{\pi_{\textnormal{ref}}(y'|x)}{\pi_{\textnormal{ref}}(y|x)}\underbrace{\exp\left( \frac{2}{\beta}\left(\delta_{P}+\frac{1+\beta\eta}{\eta}L_0+1\right)\right)}_{C_2},
\end{align*}

Summing over $y'\in\gY$ on both sides, we get
\begin{align}\label{eq:asmp_ratio}
    \forall y\in\gY:\quad \frac{1}{\pi_{\theta_{t+1}}(y|x)}\leq \frac{1}{\pi_{\theta_{\textnormal{ref}}}(y|x)}C_2.
\end{align}


Therefore, we have
\begin{align}\label{eq:bd_i}
    \left|\innerprod{\delta^{(t)}_{S,x},\pi^\star_{\beta,x}}\right| &=\sum_{y\in\gY} \frac{\pi^\star_\beta(y|x)}{\sqrt{\pi_{\theta_t}(y|x)}}\sqrt{\pi_{\theta_t}(y|x)\left(\delta^{(t)}_S(x,y)\right)^2}\notag\\
    &\leq \sqrt{\left(\sum_{y\in\gY} \frac{\left(\pi^\star_\beta(y|x)\right)^2}{\pi_{\theta_t}(y|x)}\right)\left(\sum_{y\in\gY}\pi_{\theta_t}(y|x)\left(\delta^{(t)}_S(x,y)\right)^2\right)}\notag\\
    &= \sqrt{\E_{y\sim \pi^\star_\beta(\cdot|x)}\left[\frac{\pi^\star_\beta(y|x)}{\pi_{\theta_t}(y|x)}\right]\E_{y\sim \pi_{\theta_t}(\cdot|x)}\left[\left(\delta^{(t)}_S(x,y)\right)^2\right]}\notag\\
    &\leq \sqrt{C_2\E_{y\sim \pi^\star_\beta(\cdot|x)}\left[\frac{\pi^\star_\beta(y|x)}{\pi_{\textnormal{ref}}(y|x)}\right]\E_{y\sim \pi_{\theta_t}(\cdot|x)}\left[\left(\delta^{(t)}_S(x,y)\right)^2\right]}\notag\\
    &\leq \sqrt{C_2\E_{y\sim \piref(\cdot|x)}\left[\frac{\pi^\star_\beta(y|x)}{\pi_{\textnormal{ref}}(y|x)}\right]^2\E_{y\sim \pi_{\theta_t}(\cdot|x)}\left[\left(\delta^{(t)}_S(x,y)\right)^2\right]}\notag\\
    &\leq \sqrt{C_1\E_{y\sim \pi_{\theta_t}(\cdot|x)}\left[\left(\delta^{(t)}_S(x,y)\right)^2\right]},
\end{align}
where the second line follows from the Cauchy-Schwartz inequality, and the last line uses Assumption~\ref{asmp:concentr}. 

By the same argument, we could also bound (ii):
\begin{equation}\label{eq:bd_ii}
    \left|\innerprod{\delta^{(t)}_{S,x},\pi_{\theta_{t+1},x}}\right|\leq \sqrt{C_1\E_{y\sim \pi_{\theta_t}(\cdot|x)}\left[\left(\delta^{(t)}_S(x,y)\right)^2\right]}.
\end{equation}

For term (iii), note that
$\norm{\delta^{(t)}_{P,x}}_\infty\leq \delta_P,$
where $\delta_P$ is defined in~\eqref{eq:model_error_bound}, we have
\begin{equation}\label{eq:bd_iii}
    \norm{\delta^{(t)}_{P,x}}_\infty\norm{\pi^\star_{\beta,x}-\pi_{\theta_{t+1},x}}_1\leq 2 \delta_P.
\end{equation}

Thus combining \eqref{eq:bd_i},\eqref{eq:bd_ii},\eqref{eq:bd_iii} with \eqref{eq:bound_error_pre}, we have
\begin{equation}\label{eq:bound_error}
    \innerprod{\delta^{(t)}_x,\pi^\star_{\beta,x}-\pi_{\theta_{t+1},x}}\leq 2\cdot \frac{1+\beta\eta}{\eta}\sqrt{C_1\E_{y\sim \pi_{\theta_t}(\cdot|x)}\left[\left(\delta^{(t)}_S(x,y)\right)^2\right]} + 2 \delta_P,\,\,\forall x\in\gX.
\end{equation}

Taking expectation w.r.t. $x$ on both sides of \eqref{eq:p3} and making use of \eqref{eq:bound_error}, we have
\begin{align}\label{eq:descent+error}
    &\quad\dist{\pi^\star_\beta}{\pi_{\theta_{t+1}}}\notag\\
    &=\EE_{x\sim\rho}\left[B_h(\pi^\star_{\beta,x},\pi_{\theta_{t+1},x})\right]\notag\\
    &\leq \frac{1}{1+\beta\eta}\EE_{x\sim\rho} \left[B_h(\pi^\star_{\beta,x},\pi_{\theta_{t},x})\right]+2\left( \EE_{x\sim\rho} \sqrt{C_1\E_{y\sim \pi_{\theta_t}(\cdot|x)}\left[\left(\delta^{(t)}_S(x,y)\right)^2\right]}+\frac{\eta}{1+\beta\eta}\delta_P\right)\notag\\
    &\leq \frac{1}{1+\beta\eta}\dist{\pi^\star_\beta}{\pi_{\theta_{t}}} + 2\left( \sqrt{C_1\E_{x\sim\rho,y\sim \pi_{\theta_t}(\cdot|x)}\left[\left(\delta^{(t)}_S(x,y)\right)^2\right]}+\frac{\eta}{1+\beta\eta}\delta_P\right)\notag\\
    &= \frac{1}{1+\beta\eta}\dist{\pi^\star_\beta}{\pi_{\theta_{t}}} + 2\left( \sqrt{C_1\E_{x\sim\rho,y\sim \pi_{\theta_t}(\cdot|x)}\left[\left(\delta^{(t)}_S(x,y)\right)^2\right]}+\frac{\eta}{1+\beta\eta}\delta_P\right)
\end{align}
where the second inequality follows from Jensen's inequality 
and $\delta^{(t)}_S(x,y)$.

Note that
\begin{align}\label{eq:=excess_risk}
    &\quad\,\,\E_{x\sim\rho,y\sim \pi_{\theta_t}(\cdot|x)}\left[\left(\delta^{(t)}_S(x,y)\right)^2\right] \notag\\
    &\overset{\eqref{eq:statistical_error}}{=} \E_{x\sim\rho,y\sim \pi_{\theta_t}(\cdot|x)}\left[\left(\phi_{\theta_{t+1}}(y|x)-\phi_{\theta_{t+1}^\star}(y|x)\right)^2\right]\notag\\
    &\,\,= \E_{x\sim\rho,y\sim \pi_{\theta_t}(\cdot|x)}\left[\left(\phi_{\theta_{t+1}}(y|x)+\phi_{\theta_{t+1}^\star}(y|x)-2\E_{y'\sim \pi_{\theta_t}(\cdot|x)}[ \varphi_t(x,y,y')|x,y]\right)\left(\phi_{\theta_{t+1}}(y|x)-\phi_{\theta_{t+1}^\star}(y|x)\right)\right]\notag\\
    &\,\,=\E_{x\sim\rho,y\sim \pi_{\theta_t}(\cdot|x),\atop y'\sim \pi_{\theta_t}(\cdot|x)}\left[\left(\phi_{\theta_{t+1}}(y|x)-\varphi_t(x,y,y')\right)^2-\left(\phi_{\theta_{t+1}^\star}(y|x)-\varphi_t(x,y,y')\right)^2\right]\notag\\
    &\,\,=R_t(\theta_{t+1})-R_t(\theta_{t+1}^\star)\notag\\
    &\,\,=R_t(\theta_{t+1})-R_t^\star,
\end{align}
where $R_t$ is defined in \eqref{eq:theta_t+1*}, $R_t^\star\coloneqq \min_{\theta\in\Theta}R_t(\theta)=R_t(\theta_{t+1}^\star)$, and the third line uses Assumption~\ref{asmp:expressive}. 

Combining the above expression~\eqref{eq:=excess_risk} with \eqref{eq:descent+error}, we obtain
\begin{align}\label{eq:recursion}
  \dist{\pi^\star_\beta}{\pi_{\theta_{t+1}}}\leq \frac{1}{1+\beta\eta}\dist{\pi^\star_\beta}{\pi_{\theta_t}}
+2\bigg(\underbrace{\sqrt{C_1(R_t(\theta_{t+1})-R_t^\star)}+\frac{\eta}{1+\beta\eta}\delta_P}_{\coloneqq\xi_t}\bigg).
\end{align}

The above expression implies we need to bound $\xi_t$. If for all $t$, $\xi_t$ could be bounded by some finite $\xi$, then by \eqref{eq:recursion} we have
\begin{align}\label{eq:final_bound_pre}
    \dist{\pi^\star_\beta}{\pi_{\theta_t}}&\leq \left(\frac{1}{1+\beta\eta}\right)^t \dist{\pi^\star_\beta}{\pi_{\theta_0}}+2\sum_{s=0}^{t-1}\left(\frac{1}{1+\beta\eta}\right)^s \xi\notag\\
    &\leq \left(\frac{1}{1+\beta\eta}\right)^t\dist{\pi^\star_\beta}{\pi_{\theta_0}}
    +\frac{2(1+\beta\eta)}{\beta\eta}\xi.
\end{align}

In the following, we bound $\xi_t$ by bounding the \textit{excess risk} $R_t(\theta_{t+1})-R_t^\star$.

\paragraph{Step 3: bound the excess risk.}

To bound the excess risk, we first introduce the concept of \textit{uniform stability}~\citep{bousquet2002stability}. Suppose we have a training dataset $\gD=\{z_1,\cdots,z_M\}$ where each $z_i$ is sampled i.i.d. from some unknown distribution $P$ defined on some abstract set $\gZ$. Given $\gD$, a \textit{learning algorithm} produces the decision rule $w_M=w_M(\gD)=w_M(z_1,\cdots,z_M)\in\gW$, where $\gW$ is the set of all decision rules and is assumed to be a closed subset of a separable Hilbert space. We use $w_M$ to refer to both the algorithm and the decision rule. For the loss function $\ell:\gZ\times \gW\rightarrow [0,\infty)$, we define the risk and the empirical risk of $w\in\gW$ respectively as
\begin{equation}\label{eq:risk}
    R(w)=\E_{z\sim P} \ell(z,w)\quad \text{and} \quad R_M(w)=\frac{1}{M}\sum_{i=1}^M \ell(z_i,w).
\end{equation}

\begin{definition}\label{def:uniform_stability}
    An algorithm $w_M$ is uniformly $\gamma$-stable, if for any $z,z',z_1,\cdots,z_M\in\gZ$ and any $i\in[M]$, it holds that
    \begin{equation}\label{eq:stability}
        |\ell(z,w_M(z_1,\cdots,z_M))-\ell(z,w_M(z_1,\cdots, z_{i-1}, z', z_{i+1},\cdots, x_M))|\leq \gamma.
    \end{equation}
\end{definition}

We will also use the \textit{generalized Bernstein condition} defined as follows:
\begin{definition}[Assumption~1.1 in \citet{klochkov2021stability}]\label{def:bernstein_condition}
    Define $\gW^\star\coloneqq \argmin_{w\in\gW} R(w)$ where $\gW$ is a closed set. We say that $(\gW, P, \ell)$ satisfies the generalized Bernstein condition if there exists some constant $B>0$ such that for any $w\in\gW$, there exists $w^\star\in\gW^\star$ that satisfies
\begin{equation}\label{generalized_Bernstein}
    \E_{z\in P}\left[(\ell(w,z)-\ell(w^\star,z))^2\right]\leq B(R(w)-R(w^\star)).
\end{equation}
\end{definition}
 
With the above two lemmas, we now introduce the following important lemma that bounds the generalization error for uniformly stable algorithms:
\begin{lm}[Theorem~1.1 in  \citet{klochkov2021stability}]\label{lm:generalization_bd}
    Assume loss $\ell$ is bounded by $C$ on $\gZ\times \gW$, and $(\gW, P, \ell)$ satisfies the generalized Bernstein condition with the parameter $B$~(c.f.~Definition~\ref{def:bernstein_condition}). Let $w$ be a $\gamma$-stable algorithm~(c.f.~Definition~\ref{def:uniform_stability}) that returns $w_M\in\argmin_{w\in\gW}R_M(w)$ given the training dataset $\gD$. Then with probability at least $1-\delta$, it holds that
    \begin{equation}\label{eq:generalization_error}
        R(w_M)-\inf_{w\in\gW} R(w)\leq C_r\left(\gamma\log M +\frac{C+B}{M}\right)\log\left(\frac{1}{\delta}\right),
    \end{equation}
where $C_r>0$ is an absolute constant.
\end{lm}

To proceed, we analyze the generalization error at the $t$-th iterate of Algorithm~\ref{alg:practice} for a fixed arbitrary $t\in\NN$. We'll let $\widehat\theta$ denote $\theta_{t+1}$ and drop superscript/subscript $t$ when this causes no confusion. For example, we'll simply write the update rule~\eqref{eq:regression} as
$$
\widehat\theta\leftarrow \argmin_{\theta\in\Theta} \frac{1}{M}\sum_{i=1}^M\left(\varphi(x_i,y_i,{y'}_{i})-\phi_{\theta}(y_i|x_i)\right)^2.
$$
For notation simplicity, we also let $u_i=(x_i,y_i)$, $v_i=\varphi(x_i,y_i,{y'}_{i})$, $z_i=(u_i,v_i)\in\gZ\coloneqq \gX\times\gY\times \R$, and let $\phi_\theta(u_i)$ denote $\phi_\theta(y_i|x_i)$ for all $i\in [M]$. Let $\gZ\coloneqq\gX$.
Then in our case, the loss function $\ell: \gZ\times \Theta\rightarrow \R_+$ has the following form:
\begin{equation}
    \ell (z,\theta)\coloneqq \left(v-\phi_\theta(u)\right)^2,
\end{equation}
where $z=(u,v)\in\gZ$, and similar as~\eqref{eq:risk}, our risk and empirical risk at the $t$-th itrerate satisfy:
\begin{equation}\label{eq:risk_t}
    \forall \theta\in\Theta:\quad R(\theta)=\E_{z\sim P} \ell(z,\theta)\quad \text{and} \quad R_M(\theta)=\frac{1}{M}\sum_{i=1}^M \ell(z_i,\theta),
\end{equation}
where we let $P$ denote the distribution of $z_i$ ($i\in[M]$), and we have 
\begin{equation}\label{eq:sol_risk}
    \theta^\star=\argmin_{\theta\in\Theta}R(\theta) \quad \text{and} \quad\widehat\theta=\argmin_{\theta\in\Theta}R_M(\theta).
\end{equation}

Mote that \eqref{eq:L} implies that $L(z,\theta)$ is $L$-Lipschitz over $\theta$ for any $z\in\gZ$. Then by Assumption~\ref{asmp:PL} and Remark 3 in 
\citet{kang2022sharper}, we have that $(\Theta,P,\ell)$ satisfies the generalized Bernstein condition with
\begin{equation}\label{eq:B}
    B=\frac{2L^2}{\mu}.
\end{equation}

Furthermore, Corollary 4 in \citet{charles2018stability} gives that, when Assumption~\ref{asmp:Lipschitz},\ref{asmp:PL} hold, the empirical risk $R_M$ is $\gamma$-uniform stability (c.f.~Definition~\ref{def:uniform_stability}) with 
\begin{equation}\label{eq:gamma}
    \gamma=\frac{2L^2}{\mu(M-1)}.
\end{equation}

Substituting \eqref{eq:B} and \eqref{eq:gamma} into \eqref{eq:generalization_error}, we obtain that for any fixed $t$, with probability at least $1-\delta$, we have 
\begin{equation}\label{eq:generalization_bd_ours}
    R_t(\widehat\theta)-R_t^\star\leq C_r\left(\frac{2L^2 \log M}{\mu(M-1)} +\frac{C+2L^2/\mu}{M}\right)\log\left(\frac{1}{\delta}\right).
\end{equation}

By the independence of the samples in different rounds we know that with probability at least $1-\delta$, we have
\begin{align}\label{eq:generalization_bd_ours_uniform}
    \forall t\leq T-1:\quad R_t(\widehat\theta)-R_t^\star&\leq C_r\left(\frac{2L^2 \log M}{\mu(M-1)} +\frac{C+2L^2/\mu}{M}\right)\log\left(\frac{1}{1-(1-\delta)^{1/T}}\right)\notag\\
    &\leq C_r\left(\frac{2L^2 \log M}{\mu(M-1)} +\frac{C+2L^2/\mu}{M}\right)\log\left(\frac{T}{\delta}\right).
\end{align}

\paragraph{Step 4: put everything together.} Let 
\begin{equation}\label{eq:Delta}
    \xi \coloneqq \sqrt{C_1 C_r\left(\frac{2L^2 \log M}{\mu(M-1)} +\frac{C+2L^2/\mu}{M}\right)\log\left(\frac{T}{\delta}\right)}+\frac{\eta}{1+\beta\eta}\delta_P.
\end{equation}
Then \eqref{eq:Delta}, \eqref{eq:generalization_bd_ours_uniform} and \eqref{eq:final_bound_pre} together give the desired result.

\end{document}